\documentclass{article}

\usepackage[preprint]{neurips2026}

\usepackage{amsmath}
\usepackage[utf8]{inputenc} 
\usepackage[T1]{fontenc}    
\usepackage{url}            
\usepackage{booktabs}       
\usepackage{amsfonts}       
\usepackage{nicefrac}       
\usepackage{microtype}      
\usepackage[table]{xcolor}  
\usepackage{mdframed}
\usepackage{graphicx}
\usepackage{subcaption}
\usepackage{tabularx}
\usepackage{amssymb}
\usepackage{mathtools}
\usepackage{amsthm}
\usepackage{titletoc}
\usepackage{algorithmic}
\usepackage{algorithm}
\usepackage{multirow}
\usepackage{hyperref}       
\usepackage[capitalize]{cleveref}  

\title{Margin-calibrated Classifier Guidance \\for Property-driven Synthesis Planning}

\author{%
  Najwa Laabid \\
  Department of Computer Science \\
  Aalto University \\
  \texttt{najwa.laabid@aalto.fi} \\
  \And
  Vikas Garg \\
  Department of Computer Science \\ 
  Aalto University\\ 
  YaiYai Ltd\\
  \texttt{vgarg@csail.mit.edu} \\
}

\newcommand{\rvec}{\mathbf{r}}
\DeclareMathOperator*{\argmax}{arg\,max}
\definecolor{highlightgreen}{RGB}{220,240,220}  
\newcommand{\MYCOMMENT}[1]{\hfill \textcolor{gray}{\textit{// #1}}}

\newtheorem{proposition}{Proposition}

\newtheorem{remark}{Remark}
\newtheorem*{theorem*}{Theorem}
\newtheorem*{proposition*}{Proposition 1}

\crefname{observation}{Observation}{Observations}

\begin{document}

\maketitle

\begin{abstract}
    Synthesis planning seeks an efficient sequence of chemical reactions that produce a target molecule. 
    Typically, a pretrained single-step (autoregressive) retrosynthesis model is repeatedly invoked to generate such a 
    sequence. Classifier guidance can, in principle, help steer the output of single-step model toward reactions that 
    satisfy specific constraints or accommodate chemist's preferences during inference without having to retrain the 
    autoregressive generator. We expose the insufficiency of auxiliary classifiers trained with cross-entropy loss to 
    override the unconditional token-level distributions learned from typical sparse single-disconnection reaction datasets. 
    We overcome this issue with a novel method called Sequence Completion Ranking (SCR), which employs contrastive 
    argumentation and a margin-based loss to calibrate the classifier so that it can meaningfully discriminate between 
    continuations during decoding. We formally establish that margin-calibrated classifiers can expand the set of 
    property-satisfying sequences reachable under guided beam search. Empirically, on USPTO-190, given chemist-specified guidance targets, SCR substantially
    improves multi-step solve rates from $16.8\%$ (unguided generator) to $78.4\%$ with reaction-type guidance
    and $95.3\%$ with Tanimoto guidance, unlocking valid routes for 33 targets ($17.4\%$) previously unsolvable
    with baselines. Our method also effectively closes the long-standing diversity gap between template-free and 
    template-based methods.
\end{abstract}



\section{Introduction}
Retrosynthesis, the task of predicting which precursor molecules can produce a given target, is a cornerstone of synthetic chemistry and drug discovery.
Data-driven single-step models have made remarkable progress~\citep{segler2017neuralsym,Segler2018MCTS,schwaller2019molecular,sacha2021megan,Zhong2022Rsmiles,chen2021localretro,seidl2021mhnreact}, yet chemists rarely seek \emph{any} viable precursors: a medicinal chemist may favor robust C--C cross-couplings for late-stage diversification~\citep{Brown2016ReactionsGone,Roughley2011Toolbox}, an industrial chemist may require routes through specific, low-cost starting materials, and a process chemist may avoid reactions with hazardous intermediates or pyrophoric reagents~\citep{Constable2007GreenChem}.
Current models are largely agnostic to such preferences.
Existing approaches either retrain the generator per property (expensive and brittle to constraint changes), or filter post-hoc from a fixed sample pool, leaving sequences outside the base model's beam unreachable.
Classifier guidance~\citep{Yang2021fudge,Meng2022Nado} offers a lightweight alternative: steer at inference time using an auxiliary signal, leveraging the generative prior of an existing model. \looseness=-1

We study \emph{steerable single-step retrosynthesis}: given a target molecule and a desired property (e.g.\ a reaction type or a starting material constraint), generate precursors satisfying the property while remaining chemically valid.
Our approach applies classifier guidance to an autoregressive SMILES-based model: a token-level classifier predicts whether the completed reaction will satisfy the target property and reranks candidates during beam search (\cref{eq:token-guidance}), steering generation at inference time without retraining. \looseness=-1

For this guidance to work, the classifier must assign meaningfully different probabilities to token continuations leading toward versus away from the target property, a quantity we term \emph{token-level discriminability}.
We show that classifiers trained with cross-entropy cannot guarantee this on retrosynthesis data: the data regime needed for a CE estimator to clear the discriminability gap requires per-prefix coverage of rare token--class combinations that single-disconnection reaction corpora structurally lack (\cref{sec:ce-shortcomings}).
To overcome this, we propose \emph{Sequence Completion Ranking} (SCR), combining \emph{contrastive augmentation} (replacing the final token of partial sequences with random alternatives to force token-level sensitivity) with a \emph{margin-based ranking loss} that directly enforces the discriminability gap required by guidance.
We further prove that classifiers satisfying these per-step requirements provably expand the set of sequences reachable under beam search (\cref{proposition:reachability}). \looseness=-1

We evaluate SCR on reaction type steering and starting material conditioning.
On USPTO-190-Steps (640 products), SCR achieves the highest steering breadth among all baselines including template-based methods, recovers valid steered precursors for 231 of the 287 products on which the unguided generator fails entirely, and produces samples with only 39\% Jaccard overlap with unguided generation.
In multi-step search on USPTO-190 with chemist-specified targets per node, reaction-type guidance lifts the Retro* solve rate from 16.8\% (unguided generator) to 78.4\%,
exceeding the best template-free baseline by 12.6\,pp, and Tanimoto guidance reaches 95.3\% with the highest route diversity overall.
Across both signals, guidance unlocks valid routes for 33 targets (17.4\%) that \emph{no} baseline solves, and remains competitive with search-level steering when the two are combined. \looseness=-1

\begin{itemize}
    \item We prove that cross-entropy training cannot deliver the \emph{token-level discriminability} guidance requires on retrosynthesis data, where single-disconnection corpora structurally lack the rare-cell coverage the CE estimator needs (\cref{sec:ce-shortcomings}).
    \item We propose \emph{Sequence Completion Ranking} (SCR), pairing contrastive augmentation with a margin-based ranking loss that calibrates classifiers to clear the discriminability gap by construction (\cref{sec:calibration}), and prove that SCR-trained classifiers expand the set of property-satisfying sequences reachable under guided beam search (\cref{proposition:reachability}).
    \item Empirically, on USPTO-190-Steps SCR achieves the highest steering breadth among all baselines (template-based and template-free); in multi-step search, reaction-type guidance lifts the Retro* solve rate by $4.7{\times}$ over the unguided generator and Tanimoto guidance reaches the highest solve rate and route diversity overall, jointly unlocking valid routes for 33 targets (17.4\%) that no baseline solves.
\end{itemize}

Our anonymized code release is available at \url{https://anonymous.4open.science/r/scr-submission-FB82/} and the trained classifier checkpoints are archived at \url{https://osf.io/nyazq/?view_only=0fb8e747409c4067ba2299124bcf7b46}.

\begin{figure}[t]
    \centering
    \resizebox{0.60\textheight}{!}{
    \includegraphics[width=0.9\textwidth]{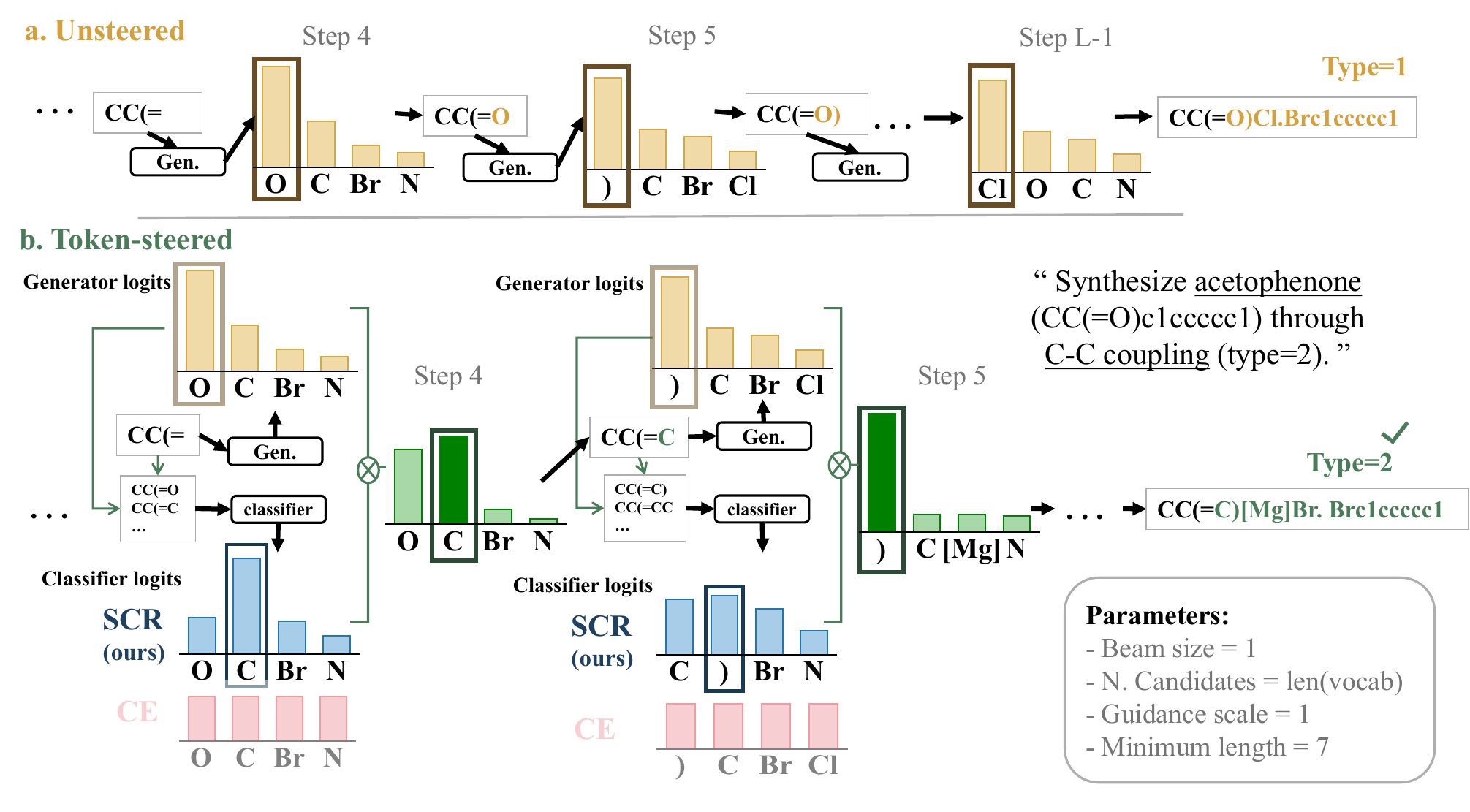}
    }
\caption{Autoregressive generation under three configurations: unguided generator, generator + CE-trained classifier, and generator + SCR-trained classifier. CE-guided decoding produces the same outcome as the unguided generator (\cref{sec:ce-shortcomings}), whereas SCR steers generation toward the target reaction type (C--C coupling).}
\label{fig:token-reranking}
\end{figure}

\section{Preliminaries}
\label{sec:preliminaries}

\subsection{Single-step retrosynthesis}
\label{sec:single-step-retro}

Single-step retrosynthesis is the task of predicting plausible precursor molecules (reactants) for a given target product.
Formally, we define $\mathcal{M}$ as the space of all valid molecules and represent a reaction as a pair $R = (\rvec, p)$, 
where $\rvec \subset \mathcal{M}$ is the set of reactants and $p \in \mathcal{M}$ is the product.
A template-free single-step model parameterizes the conditional distribution $p_\theta(\rvec \mid p)$,
generating reactant predictions directly without relying on predefined reaction templates. \looseness=-1

Following the protocol of \citet{Zhong2022Rsmiles}, we train our base generator with root-aligned SMILES, where products and reactants
are aligned to minimize their edit distance so that the underlying autoregressive model focuses on learning local chemical transformations.
Both products and reactants are represented as SMILES strings, and the generator emits the reactant string token by token,
conditioned on the product.
At inference we follow the augmentation protocol of \citet{Zhong2022Rsmiles}; details in \cref{app:lookahead-inference}.
Throughout the paper, we refer to this product-reactant-aligned autoregressive model simply as the \emph{generator}. \looseness=-1

In many practical settings, chemists seek not just \emph{any} valid set of precursors, but precursors satisfying 
additional constraints, such as proceeding via a specific reaction type or incorporating a particular starting material.
We refer to this as \emph{steerable} single-step retrosynthesis.
Prior work on steerable synthesis planning operates primarily at the search level in multi-step settings,
filtering single-step predictions \citep{Segler2018MCTS} or modifying value functions \citep{Armstrong2024TangoStar}.
At the single-step level, recent work steers retrosynthesis language models through disconnection prompts \citep{Thakkar2023Disconnection}
or LLM agents that wrap chemistry tools \citep{Bran2024ChemCrow}; both operate above the token level and inherit the base model's coverage. \looseness=-1

\subsection{Classifier guidance in autoregressive models}
\label{sec:classifier-guidance}

Let $\{\phi_\alpha\}_{\alpha=1}^n$ be a family of property predicates, with each $\phi_\alpha(\rvec, p) \in \{0,1\}$
indicating whether the reaction $(\rvec, p)$ satisfies property $\alpha$; we drop $(\rvec, p)$ when clear from context, writing $\phi_\alpha$.
For each $\alpha$ we learn a classifier $p_{\alpha}(\phi_\alpha = 1 \mid \rvec, p)$ that estimates $P(\phi_\alpha = 1 \mid \rvec, p)$.
Given this classifier and the pre-trained generator $p_\theta(\rvec \mid p)$, classifier guidance \citep{nichol2021diffusionmodelsbeatgans, Yang2021fudge} defines
the guided distribution via Bayes' rule:
\begin{equation}
    \label{eq:tempered-guided-distribution}
    p_{\theta}^{\lambda}(\rvec \mid p, \phi_\alpha = 1)
    \propto \big[p_{\alpha}(\phi_\alpha = 1 \mid \rvec, p)\big]^{\lambda} \cdot p_\theta(\rvec \mid p),
\end{equation}
where $\lambda \in \mathbb{R}^+$ is a guidance scale controlling the strength of the property signal. \looseness=-1

For autoregressive generators, the guided distribution factorizes into a token-level form. Let $r_k \in \mathcal{V}$ denote the $k$-th token in the reactant sequence $\rvec$, with $\mathcal{V}$ the vocabulary:
\begin{equation}
    \label{eq:token-guidance}
    p_{\theta}^{\lambda}(r_k \mid \rvec_{<k}, p, \phi_\alpha = 1)
    \propto \big[p_{\alpha}(\phi_\alpha = 1 \mid \rvec_{<k}, r_k, p)\big]^{\lambda} \cdot p_\theta(r_k \mid \rvec_{<k}, p).
\end{equation}
At each decoding step, the classifier scores every candidate next token $r_k \in \mathcal{V}$ appended to the current prefix $\rvec_{<k}$, and these scores are combined with the generator's logits before beam expansion (\Cref{fig:token-reranking}). The classifier $p_{\alpha}(\phi_\alpha = 1 \mid \rvec_{<k}, r_k, p)$ must therefore estimate the property of the \emph{complete} reaction from a partial sequence $\rvec_{\leq k} = (\rvec_{<k}, r_k)$. \looseness=-1

\section{Methods}
\label{sec:methods}

Classifier guidance reranks tokens by adding classifier log-probabilities to generator logits (\cref{eq:token-guidance}).
Let $\mathcal{R}_\phi = \{\rvec \in \mathcal{V}^L \mid \phi_\alpha(\rvec, p) = 1\}$ denote the set of all length-$L$
property-satisfying sequences, and $\mathrm{Beam}_\omega$ the set of sequences returned by unguided beam search of width $\omega$.
Fix any $\rvec^* \in \mathcal{R}_\phi \setminus \mathrm{Beam}_\omega$, a property-satisfying sequence absent from the unguided beam,
and let $\hat r_k = \arg\max_{r_k} p_\theta(r_k \mid \rvec^*_{<k}, p)$ be the generator's top-scoring token at step~$k$ along $\rvec^*$.
Recovering $\rvec^*$ under beam search requires that, at every step, the guided score of the target token $r^*_k$ exceed that of $\hat r_k$.
Substituting the guided distribution (\cref{eq:token-guidance}) gives the per-step necessary condition:
\begin{equation}
\underbrace{\log \frac{p_\alpha(\phi_\alpha \mid \rvec^*_{<k}, r^*_k, p)}{p_\alpha(\phi_\alpha \mid \rvec^*_{<k}, \hat r_k, p)}}_{\Delta_k}
\;>\; \frac{1}{\lambda}
\underbrace{\log \frac{p_\theta(\hat r_k \mid \rvec^*_{<k}, p)}{p_\theta(r^*_k \mid \rvec^*_{<k}, p)}}_{G_k}
\label{eq:gap-condition}
\end{equation}
We call the left-hand side the \emph{token-level discriminability} $\Delta_k$:
the classifier's log-probability gap between the target continuation $r^*_k$
and the generator's preferred token $\hat r_k$. The right-hand side $G_k$ is
the \emph{guidance requirement}: the minimum discriminability needed to
overcome the generator's preference at step $k$. \looseness=-1
 
The standard recipe for training such a token-level classifier is
cross-entropy on partial sequences, optionally augmented with samples from
the generator to improve coverage at decoding-time prefixes
\citep{Yang2021fudge,Meng2022Nado}. We show that cross-entropy estimation
cannot produce a $\Delta_k$ satisfying Eq.~\eqref{eq:gap-condition} on
retrosynthesis data: doing so requires a per-prefix sample regime that
reaction corpora structurally do not realize.
Sec.~\ref{sec:ce-shortcomings} makes this concrete through a
sample-complexity analysis on a toy model and translates the data
requirement back into reaction terms; \emph{Sequence Completion Ranking} (SCR; Sec.~\ref{sec:calibration})
circumvents the estimation entirely by installing the discriminability
gap by construction. \looseness=-1
 
\subsection{Why cross-entropy training fails on reaction data}
\label{sec:ce-shortcomings}
\begin{figure}[tbh]
    \centering
    \begin{subfigure}[t]{0.42\linewidth}\centering
        \vspace{0pt}
        \includegraphics[width=\linewidth]{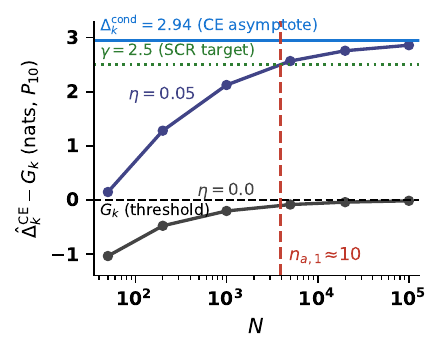}
        \subcaption{Empirical CE signal vs.\ per-prefix samples.}\label{fig:ce-failure-modes-curves}
    \end{subfigure}\hfill
    \begin{subfigure}[t]{0.55\linewidth}\centering
        \vspace{0pt}
        \includegraphics[width=\linewidth]{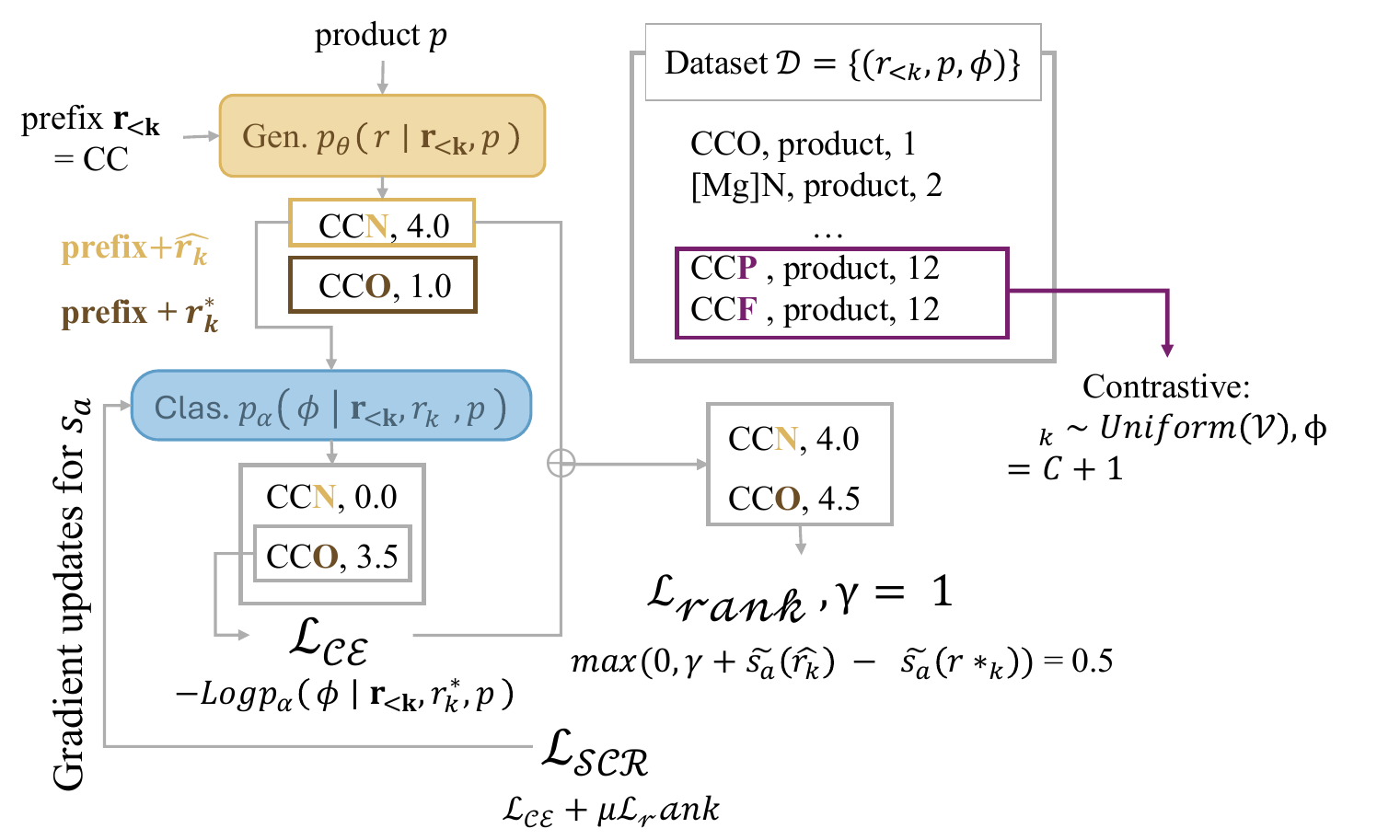}
        \subcaption{SCR training pipeline.}\label{fig:ce-failure-modes-scr}
    \end{subfigure}
    \caption{\textbf{CE training fails to deliver token-level discriminability on retrosynthesis data; SCR installs the gap by construction.}
    \textbf{(a)}~Excess CE signal $\hat\Delta_k^{\mathrm{CE}}-G_k$ vs.\ per-prefix samples $N$ (at $\epsilon{=}0.05$); 
    curves are the 10th percentile of $\hat\Delta_k^{\mathrm{CE}}$ across simulated training draws, the value reached in 90\% of runs, 
    matching $\delta{=}0.1$ in the bound. The $\eta{=}0$ (USPTO) curve stays at $G_k$ for any $N$; at $\eta{=}0.05$ the signal clears 
    $\gamma$ only at $\mathbb{E}[n_{a,1}]\approx 10$ rare-cell examples per prefix (red dashed), unattainable per-prefix on retrosynthesis data.
    \textbf{(b)}~The SCR training pipeline; ranking loss decouples discriminability from per-prefix sample count.}
    \label{fig:ce-failure-modes}
\end{figure}

\paragraph{What guidance demands, in a clean toy model.}
Eq.~\eqref{eq:gap-condition} requires the classifier to deliver $\Delta_k > G_k/\lambda$
at every step where target and generator disagree. To see what this asks of the
data, consider a simple toy example: a single prefix $X$,
a binary vocabulary $\{a,b\}$, and a binary class label $\phi \in \{0,1\}$,
with class~$1$ the steering target.
Class~$0$ prefers $a$, class~$1$ prefers $b$. Let
$\epsilon = P(t{=}a \mid \phi{=}1, X) = P(t{=}b \mid \phi{=}0, X)$ be the
within-class probability of the dispreferred token, and
$\eta = P(\phi{=}1 \mid X)$ the prevalence of the target class at $X$.
When the target class is the minority at $X$ ($\eta < 1/2$), the generator
(trained on the same distribution) prefers $a$, so the classifier must
overturn it: $r_k^* = b$, $\hat{r}_k = a$. Despite its simplicity, this
toy captures the structural feature that drives the failure on reaction
data: a class--token combination that is rare relative to the generator's prior. \looseness=-1

\paragraph{The CE estimator concentrates on a single rare cell.}
A CE-trained classifier targets the class posterior
$q_t := P(\phi{=}1 \mid X, t)$, so its discriminability between the two
candidate tokens is $\Delta^{\mathrm{CE}}_k = \log q_b - \log q_a$. With $N$ i.i.d.\ samples at $X$, the $2{\times}2$ token-by-class counts $n_{t,\phi}$
(number of samples in cell $(t,\phi)$) give the plug-in estimate
$\hat{\Delta}^{\mathrm{CE}}_k = \log \hat{q}_b - \log \hat{q}_a$ (the empirical
$\Delta_k$ of \cref{eq:gap-condition}), with
$\hat{q}_t = n_{t,1}/(n_{t,0}+n_{t,1})$.
$\hat{\Delta}^{\mathrm{CE}}_k$ is a noisy estimate of $\Delta^{\mathrm{CE}}_k$,
so \cref{eq:gap-condition} holds at a given prefix only when
$\hat{\Delta}^{\mathrm{CE}}_k$ exceeds $G_k$ with high probability, as 
set by its sampling variance. A delta-method calculation
(Appendix~\ref{app:sample-complexity}) shows that, when $q_a \ll 1$,
\[
    \mathrm{Var}(\hat{\Delta}^{\mathrm{CE}}_k)
    \;\approx\; \frac{1}{\mathbb{E}[n_{a,1}]},
    \qquad \mathbb{E}[n_{a,1}] \,=\, N\eta\epsilon.
\]
The variance is governed by $\mathbb{E}[n_{a,1}]$: the expected count of training
examples that belong to the steering target class but follow the
generator-preferred (and, from the target's perspective, ``wrong'') token. We
call this the \emph{rare cell}. It shrinks whenever the steering target is
uncommon ($\eta$ small) or class--token preferences are sharp ($\epsilon$
small), precisely the regimes where guidance has the most to contribute. \looseness=-1

\paragraph{From variance to a per-prefix sample budget.}
For a tolerated failure probability $\delta$, the per-prefix success condition
reads $\Pr[\hat{\Delta}^{\mathrm{CE}}_k > G_k] \ge 1{-}\delta$. A normal
approximation to $\hat{\Delta}^{\mathrm{CE}}_k$, using the variance above,
turns this into
\[
    \mathbb{E}[n_{a,1}] \;\ge\; z_{1-\delta}^2 \big/ \bigl(\Delta^{\mathrm{cond}}_k\bigr)^2,
\]
where $\Delta^{\mathrm{cond}}_k := \Delta^{\mathrm{CE}}_k - G_k$ is the gap
between the CE signal and $G_k$ (derivation in
Appendix~\ref{app:sample-complexity}). At a target gap of
$\Delta^{\mathrm{cond}}_k = 1$, the bound gives $\mathbb{E}[n_{a,1}] \gtrsim 10$
rare-cell examples \emph{per prefix} (\cref{fig:ce-failure-modes-curves},
red dashed line); below that, sampling fluctuations regularly push
$\hat{\Delta}^{\mathrm{CE}}_k$ below $G_k$. \looseness=-1

\paragraph{Why this is unattainable on retrosynthesis data.}
The combination ``common prefix, target class, generator-dispreferred token''
(i.e. the the rare cell), is sparse in the training data, even when augmented with
on-policy samples from $p_\theta$ relabeled via RXN-Insight \citep{rxn_insight}
(the source of nonzero unguided steering breadth). So $N\eta\epsilon$ stays well
below the threshold above, and the CE estimate collapses toward $G_k$ regardless
of the total number of samples $N$ (\cref{fig:ce-failure-modes-curves}).
SCR's contrastive augmentation (\cref{sec:calibration}) supplies synthetic rare-cell samples at every prefix the decoder visits, and the ranking loss installs the discriminability gap by construction. \looseness=-1

\subsection{Calibrating for discriminability}
\label{sec:calibration}
 
\Cref{sec:ce-shortcomings} establishes that CE training cannot
estimate $\Delta_k^{CE}$ at the resolution guidance requires.
SCR addresses this through two complementary modifications, each
targeting one of the failure modes identified above. \looseness=-1
 
\textbf{Contrastive augmentation: populating the rare cell.}
For each partial sequence $\rvec_{\le k}$, we construct a same-length negative
example by replacing only its final token $r_k$ with a uniformly random
vocabulary token $\tilde r_k$ (leaving the prefix $\rvec_{<k}$ untouched),
and assign it a catch-all label $y = C+1$ distinct from all $C$ reaction
classes. Uniform sampling is what keeps the label safe: a generator-plausible
$\tilde r_k = \hat r_k$ is often valid under some off-target class
$y' \neq y$, so a $C+1$ label there would suppress $y'$ at inference. This
populates the rare cell at off-policy positions, providing the cross-class
boundary signal that single-class-per-product annotation denies the
classifier. \looseness=-1

\textbf{Ranking loss: enforcing the margin directly.}
Off-policy augmentation alone leaves two gaps: the resulting CE objective
fits a posterior $p_{\alpha}(\phi \mid \rvec_{\le k}, p)$ whose log-odds at
$(r^*_k, \hat r_k)$ are not constrained to exceed any particular threshold,
and the on-policy position $\hat r_k$ itself remains unreachable through CE
without the label leakage above. We add a margin-based ranking loss that
closes both: enforcing only \emph{relative} ordering commits to no class
label, so it can safely use the generator's top alternative
$\hat r_k = \arg\max_{r \neq r^*_k} p_\theta(r \mid \rvec_{<k}, p)$ as the
negative anchor. For each training sequence we require the
\emph{guided log-score}
\begin{equation}
\tilde s_\alpha(r) \;:=\; \log p_\theta(r \mid \rvec_{<k}, p)
\;+\; \log p_\alpha(\phi \mid \rvec_{<k}, r, p)
\label{eq:guided-score}
\end{equation}
at the ground-truth next token $r^*_k$ to exceed that of $\hat r_k$
by at least margin $\gamma$:
\begin{equation}
\tilde s_\alpha(r^*_k) - \tilde s_\alpha(\hat r_k) \;\geq\; \gamma.
\label{eq:ranking}
\end{equation}
The guided score mirrors the inference-time decoder, so the rank loss
calibrates exactly the quantity that beam search consumes; in the
notation of \cref{sec:ce-shortcomings} it bounds the useful gap
$\Delta_k - G_k$ by $\gamma$.
Augmentation supplies broad off-policy coverage; the
ranking loss installs a calibrated margin at the on-policy positions
guidance must clear, making the two complementary. \looseness=-1

\textbf{Combined objective.}
The classifier $p_\alpha$ is trained on the union of ground-truth partial sequences, contrastive augmentations, and on-policy samples from $p_\theta$ (jointly denoted $\mathcal{D}$), minimizing
\begin{equation}
\mathcal{L} \;=\; \mathbb{E}_{\mathcal{D}}\!\bigl[\mathcal{L}_{\mathrm{CE}}(p_\alpha, y)\bigr]
\;+\; \mu\, \mathbb{E}_{\mathcal{D}}\!\bigl[\max\bigl(0,\, \gamma + \tilde s_\alpha(\hat r_k) - \tilde s_\alpha(r^*_k)\bigr)\bigr],
\label{eq:scr-loss}
\end{equation}
where $\gamma > 0$ and $\mu > 0$ are hyperparameters. The training
procedure is summarized in Algorithm~\ref{alg:app-classifier-training}. \looseness=-1

\subsection{From token discriminability to property reachability}
\label{sec:reachability}

The preceding analysis establishes \emph{what} the classifier must achieve at each decoding step.
We now show that a classifier satisfying these per-step requirements provably expands the set of property-satisfying sequences reachable under guided decoding, recovering targets that unguided decoding would discard.
The ranking margin $\gamma$ enforced during training translates, after softmax, into classifier probability bounds $c_1 > c_2$ (\cref{app:chemical-space-exploration}); these bounds drive the reachability guarantee below. \looseness=-1

\begin{proposition}{\rm \bf (Reachability under Guided Beam Search)}
    \label{proposition:reachability}
    Let $\beta_n^{(l)}$ be the score of the $n$-th beam (i.e. the beam with
    the lowest accepted score) at generation step $l$, and assume
    $\mathcal{R}_\phi \setminus \text{Beam}_\omega \neq \emptyset$.
    Let $\rvec^* \in \mathcal{R}_\phi \setminus \text{Beam}_\omega$ be any
    property-satisfying sequence absent from the unguided beam.
    Assume the classifier $p_{\alpha}$ satisfies for all $l \in [1, L]$:
    \begin{align}
    \label{eq:classifier-bounds-main}
        p_{\alpha}(\phi_\alpha &= 1 \mid \rvec^*_{\leq l}, p) \geq c_1 \\
        p_{\alpha}(\phi_\alpha &= 1 \mid \rvec_{\leq l}, p) \leq c_2,
        \quad \text{when } \phi_{\alpha}(\rvec, p) = 0 \text{ and } \rvec_{\leq l} \neq \rvec^*_{\leq l}
    \end{align}
    where $0 < c_2 < c_1 \leq 1$. The upper bound is required only at
    prefixes that diverge from $\rvec^*$; at shared prefixes the
    classifier output is unconstrained, since both sequences receive the
    same per-step contribution and these cancel in any beam comparison.
    Then:
    \begin{enumerate}
        \item Unguided beam search ($\lambda = 0$) excludes $\rvec^*$ from
        the final output,
        \item There exists $\lambda^* > 0$ such that guided beam search
        with guidance scale $\lambda \geq \lambda^*$ includes at least one
        property-satisfying sequence in the final output.
    \end{enumerate}
\end{proposition}
The classifier bounds $c_1$ and $c_2$ translate into a concrete guidance
scale $\lambda^*$ above which new property-satisfying sequences become
reachable. Larger $\gamma$ widens the gap $c_1 - c_2$, which in turn
lowers the required $\lambda^*$. Proofs and further discussion are given in
\cref{app:theoretical-analysis,app:chemical-space-exploration}. \looseness=-1

Since $\lambda^*$ varies per target, at inference we select $\lambda$ adaptively via a lookahead procedure (\cref{app:lookahead-inference}).

\section{Experiments}
\label{sec:experiments}
We evaluate SCR along three axes:
(1) the contribution of each training component to classifier quality,
(2) whether token-level guidance explores different chemical space in single-step retrosynthesis, and
(3) whether single-step guidance improves multi-step route finding.
We use reaction type and starting material similarity as guidance targets
(\cref{app:properties-additional}). \looseness=-1

\paragraph{Experimental setup.}
\label{sec:experimental-setup}
We evaluate single-step performance on USPTO-190-Steps (640 deduplicated test products of USPTO-190~\citep{chen2020retrostar})---an upper-bound proxy for multi-step performance, since precursors not recovered here also fail downstream---and multi-step search on the full 190 USPTO-190 targets; USPTO-50k results are in \cref{app:additional-results}.
The base generator (\cref{sec:single-step-retro}) runs with beam size~10; following the RSMILES protocol \citep{Zhong2022Rsmiles} we augment each product with 20 canonical-SMILES randomizations and run beam search per augmentation, pooling and deduplicating the candidates to obtain 100 samples per product (also the syntheseus default for template-free models). All other methods generate 100 samples per product to match this budget.
We use the SCR-trained reaction-type classifier and a Tanimoto-similarity predictor that swaps the classification head for an MSE regression head with the same training procedure (\cref{app:tanimoto-data}).
Baselines are all methods in Syntheseus~\citep{maziarz2023syntheseus} plus NeuralSym~\citep{segler2017neuralsym}; template-based methods are marked $^\mathcal{T}$.
For multi-step search, the guidance target at each node is set to the corresponding ground-truth value (reaction class or starting material) from the reference route, simulating a chemist who specifies the desired disconnection; we report a strict apples-to-apples comparison applying the same target to baselines via post-hoc filtering in \cref{tab:search-filtered-baselines}.
Generated precursors are validated with RDKit, named by RXN-Insight~\citep{rxn_insight}, and checked for round-trip accuracy via the forward model of \citet{Zhong2022Rsmiles}.
We report \emph{steering breadth} (fraction of non-ground-truth reaction types where at least one sample is class-correct, per product), with stricter variants requiring RXN-Insight reaction naming (+RXN-I) or round-trip validity (+RT), top-$k$ accuracy, and Jaccard overlap between guided and unguided samples.
Full setup details are in \cref{app:datasets-and-processing-additional,app:experimental-setup,app:evaluation}. \looseness=-1

\paragraph{Classifier training ablation}
\label{sec:training-property-predictors-experiments}
To understand the contribution of each SCR component (\cref{sec:calibration}),
we train reaction type classifiers with different combinations of training data
(ground-truth vs.\ on-policy generator samples), contrastive augmentation,
and ranking loss, evaluating each on a held-out set of 1M partial sequences (\cref{app:experimental-setup}).
\Cref{fig:src-ablation} shows that training on generator samples with the ranking loss
(red) achieves the highest accuracy. Removing the ranking loss (orange) reduces accuracy,
confirming its role in producing discriminative logits. Training on ground-truth sequences alone (purple, gold) yields lower
accuracy due to reduced diversity; adding the ranking loss to ground-truth training (gold) causes overfitting on
early tokens, as the generator already ranks these correctly, leaving minimal ranking signal (panel~c).
Omitting wrong-token augmentation (dashed gray) degrades performance throughout.
Subsequent experiments use the SCR classifier corresponding to the red line, with $L_{\min}{=}5$, $\lambda{=}1$ for multi-step search, and lookahead over $\lambda \in \{0, 0.5, 1\}$ for single-step (\cref{app:additional-results-single-step}). \looseness=-1
\begin{figure*}[!t]
    \centering
    \includegraphics[width=\textwidth]{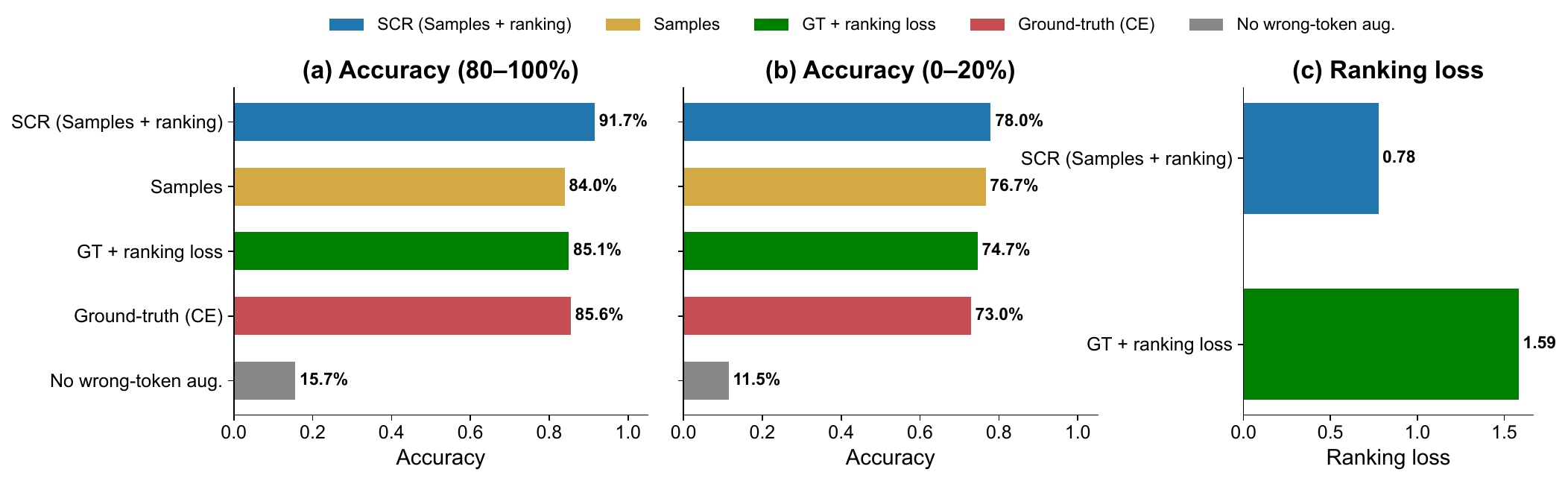}
    \caption{Ablations of classifier training methods on reaction type prediction.
    Accuracy is evaluated on partial sequences from generator samples augmented with wrong-token completions,
    with (a) showing later tokens (80\% complete) and (b) earlier tokens (0--20\% complete).
    Training on generator samples with ranking loss (red) achieves highest accuracy.
    Training on ground-truth sequences with ranking loss (gold) overfits on early tokens, as the generator
    already ranks these correctly---yielding minimal ranking signal (c). Dashed gray: no wrong-token augmentation.}
    \label{fig:src-ablation}
\end{figure*}

\paragraph{Single-step retrosynthesis}
\label{sec:single-step-evaluation}
We guide each product in USPTO-190-Steps toward every available reaction type and
assess whether guidance (1)~explores different chemical space, (2)~increases property
satisfaction while preserving validity, and (3)~improves ground-truth recovery. When steering toward non-ground-truth reaction types, SCR$_{\text{rxn}}$ outperforms all baselines,
including template-based methods.
The per-class breakdown (\cref{fig:delta-heatmaps}) makes this concrete: when guiding products normally produced by Heteroatom Alkylation toward C--C Coupling, the guided model solves 27 more products than the unguided generator and 33 more than MEGAN.
Overall, SCR$_{\text{rxn}}$ achieves a steering breadth of 0.64, surpassing the best
template-based method (RetroKNN at 0.51) and closing a long-standing gap between template-free and
template-based diversity.
When additionally requiring round-trip validity or RXN-Insight reaction naming,
SCR$_{\text{rxn}}$ matches the best template-based methods
(0.25 vs.\ 0.24 for +RT; 0.39 vs.\ 0.41 for +RXN-I),
confirming that steered samples are not only class-correct but also chemically feasible
(\cref{fig:steering-breadth}).
This is also reflected in top-$k$ accuracy improvements (\cref{app:additional-results}).
Crucially, guided and unguided samples exhibit only 39\%
Jaccard overlap (\cref{fig:jaccard}), confirming genuine exploration of new chemical space
rather than simple reranking.
Among 287 products where the unguided generator fails entirely, guidance recovers 113 via exact match and 231 via round-trip accuracy (\cref{fig:failure}). Further analysis is in \cref{app:additional-results}. \looseness=-1
\begin{figure*}[htbp]
    \centering
    \resizebox{0.60\textheight}{!}{
    \includegraphics[width=\textwidth]{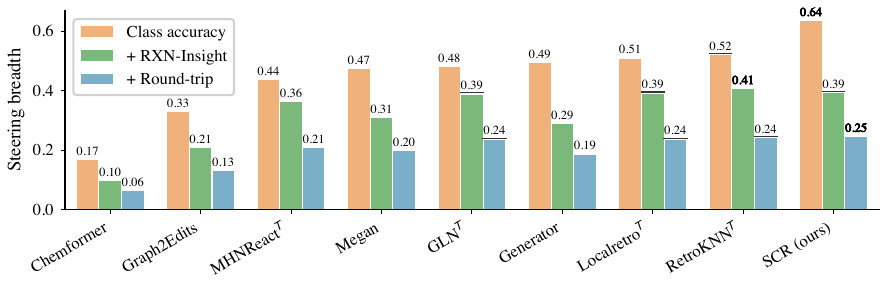}
}
\caption{Steering breadth for all baselines. SCR$_{\text{rxn}}$ achieves the highest steering breadth
on class accuracy and matches template-based methods when additionally filtering for round-trip validity (+RT) or RXN-Insight reaction naming (+RXN-I).}
\label{fig:steering-breadth}
\end{figure*}
\paragraph{Multi-step synthesis}
We evaluate whether single-step guidance improves multi-step route finding
using Retro* search~\citep{chen2020retrostar} on the 190 target molecules of USPTO-190
with a budget of 100 model calls per target (\cref{tab:search-based-synthesis-reaction-type-quality}).
When token-level and search-level guidance are combined (SG$_{\text{rxn}}$, \cref{app:additional-results}), the Retro* value function is additionally biased toward precursors whose predicted property matches the target, with a multiplicative factor balancing synthesizability against the steering signal.
With a USPTO-50k checkpoint, SCR$_{\text{rxn}}$ delivers a $4.7{\times}$ improvement in solve rate over the unguided generator and surpasses the best template-free baseline (Megan).
SCR$_{\text{tan}}$ achieves the highest solve rate overall and the highest route diversity among all methods.
Tanimoto guidance steers toward similar starting materials without enforcing exact match, so the Rate~w/~SM column trails the raw solve rate---a conceptual feature of similarity-based steering rather than a failure of the method.
Combining token-level and search-level steering (SG$_{\text{rxn}}$) yields a comparable solve rate to S$_{\text{rxn}}$ alone with a marginal gain in route diversity (\cref{tab:search-metrics-combined}); we discuss this trade-off in \cref{app:additional-results}.
On Pistachio Reachable and Pistachio Hard, Tanimoto guidance reaches the top solve rates (\cref{tab:search-pistachio-reachable,tab:search-pistachio-hard}). \looseness=-1
\begin{table}[htb!]
\caption{Search-based synthesis planning on USPTO-190 with Retro*~\citep{chen2020retrostar}, 100 model calls per target, 50k and 190 single-step checkpoints.
\textit{Rate}: fraction of targets with at least one valid route. \textit{Rate w/ SM}: fraction whose route contains the conditional starting material (Tanimoto guidance only). \textit{N. Routes}: average non-overlapping routes per solved target. \textit{R.Trip}/\textit{Class}/\textit{Name}: round-trip accuracy, ground-truth class match, and RXN-Insight name match averaged over routes. \textit{Nodes}/\textit{Calls}: average search nodes and single-step model calls per target ($\downarrow$ better). Bold/underlined: best/second-best per checkpoint. Green rows are SCR.}
\label{tab:search-based-synthesis-reaction-type-quality}
\centering
\small
\setlength{\tabcolsep}{2pt}
\begin{tabularx}{\columnwidth}{@{}Xccccccccc@{}}
    \toprule
    Method & Rate & Rate w/ SM & N. Routes & R.Trip & Class & Name & Nodes & Calls \\
     & ($\uparrow$) & ($\uparrow$) & ($\uparrow$) & ($\uparrow$) & ($\uparrow$) & ($\uparrow$) & ($\downarrow$) & ($\downarrow$) \\
    \midrule
    \multicolumn{9}{l}{\textit{USPTO-50k}} \\[0pt]
    \quad Retroknn$^{\mathcal{T}}$ & 0.54 & \underline{0.25} & 1.58 & 0.23 & 0.14 & 0.37 & 8013.27 & 37.64 \\
    \quad Megan & 0.66 & 0.21 & 1.37 & \textbf{0.27} & 0.13 & \textbf{0.43} & 8505.98 & 34.77 \\
    \quad Generator & 0.17 & 0.01 & 0.26 & 0.06 & 0.02 & 0.08 & 2410.82 & 33.76 \\
    \rowcolor{highlightgreen}
    \quad SCR$_{\text{rxn}}$ & 0.78 & 0.15 & 2.41 & 0.18 & 0.22 & 0.31 & \underline{1023.66} & 7.00 \\
    \rowcolor{highlightgreen}
    \quad SCR$_{\text{tan}}$ & \textbf{0.95} & 0.18 & 3.33 & 0.23 & \textbf{0.24} & \textbf{0.43} & \textbf{876.99} & \underline{6.64} \\[0pt]
    \multicolumn{9}{l}{\textit{USPTO-190}} \\[0pt]
    \quad Generator & 0.81 & 0.17 & 1.93 & \underline{0.26} & \underline{0.23} & \underline{0.42} & 1071.63 & 9.14 \\
    \rowcolor{highlightgreen}
    \quad SCR$_{\text{rxn}}$ & \underline{0.91} & 0.17 & \underline{4.40} & 0.05 & 0.19 & 0.05 & 1436.35 & \textbf{5.58} \\
    \rowcolor{highlightgreen}
    \quad SCR$_{\text{tan}}$ & \underline{0.91} & \textbf{0.35} & \textbf{8.22} & 0.10 & 0.20 & 0.14 & 3613.41 & 9.82 \\
    \bottomrule
\end{tabularx}
\end{table}

\paragraph{Ablations}
\label{sec:ablations}
We ablate two key hyperparameters of guided decoding, keeping the others fixed
at their defaults ($\lambda{=}1.0$, $N{=}72$, $L_{\min}{=}5$):
the guidance scale $\lambda$ and the guidance onset length $L_{\min}$
(tokens generated before guidance begins).
The number of candidate tokens $N$ has virtually no effect
($< 0.2$\,pp variation across $N{=}10$--$72$), so we fix $N{=}72$ and omit this ablation.
\Cref{fig:ablation-combined} summarizes the results; full per-setting tables are in \cref{app:ablations}.
Applying guidance earlier ($L_{\min}$) monotonically improves class accuracy by ${\sim}3.7$\,pp
without degrading round-trip accuracy: the classifier extracts useful signal from the product alone, achieving above-chance accuracy even at the first reactant tokens (\cref{fig:src-ablation}(b)), so earlier onset is a net benefit.
Increasing $\lambda$ improves class accuracy by $7.7$\,pp at the cost of a $6.4$\,pp
drop in round-trip accuracy, motivating the adaptive lookahead algorithm (\cref{app:lookahead-inference}).
Additional experiments justifying the need for token-level guidance
(vs.\ wider beams or more samples) are given in \cref{app:token-guidance-needed-additional}. \looseness=-1

\paragraph{Case study.}
We illustrate token-level guidance on Wieland-Gumlich Aldehyde, a central intermediate in known Strychnine syntheses \citep{Genheden2025RouteSimilarity}, where reaction-type guidance recovers a ground-truth precursor inaccessible to the unguided generator (\cref{fig:case-study}a).
Targets~49 and~169 are unsolved by \emph{all} baselines; rxntype and Tanimoto guidance respectively discover valid 3-step routes (\cref{fig:case-study}b,c).
Across the full dataset, rxntype and Tanimoto guidance jointly unlock routes for 33 out of 190 targets (17.4\,\%) that remain inaccessible to all baselines.
These cases illustrate the practical value of our reachability guarantee
(\cref{proposition:reachability}): the classifier's token-level signal rescues
property-satisfying sequences that would otherwise fall below the beam threshold. \looseness=-1
\begin{figure*}[!t]
    \centering
    \includegraphics[width=\textwidth]{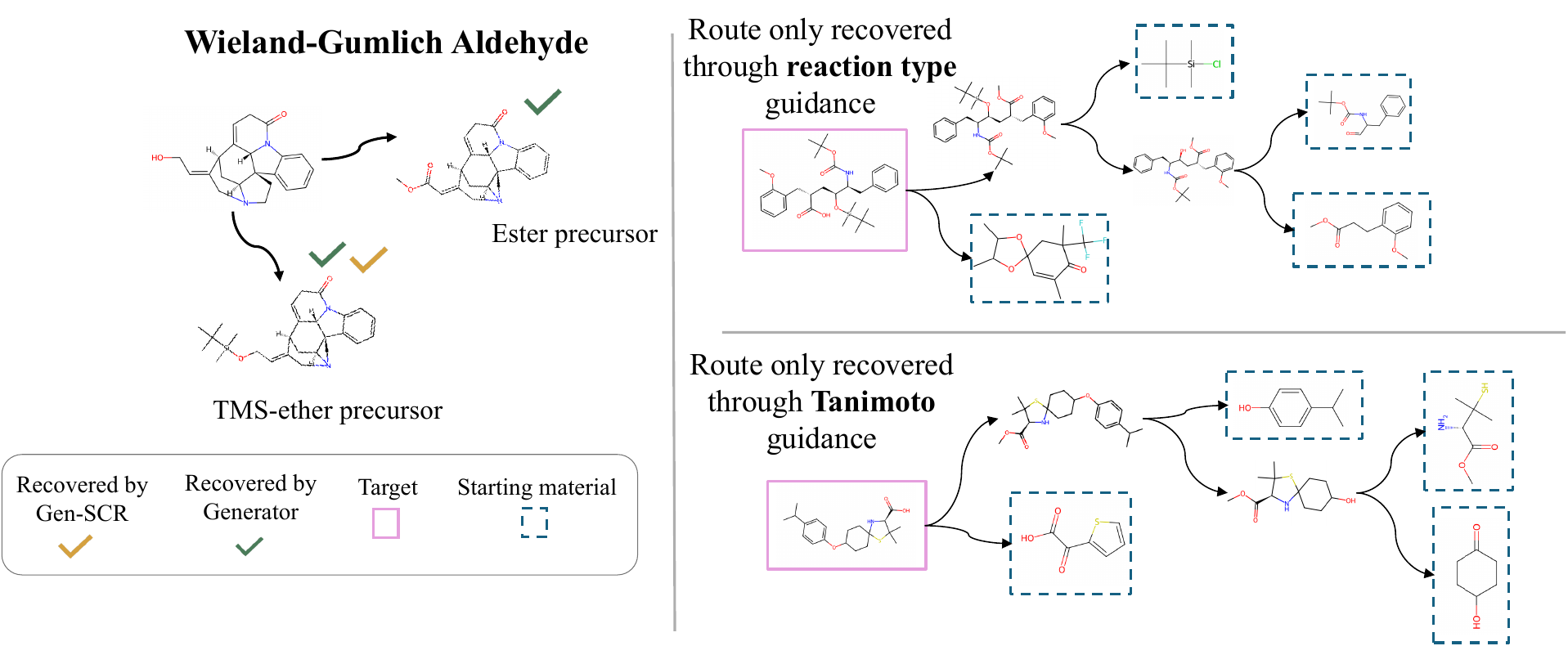}
    \caption{Case studies on USPTO-190. \textbf{(a)}~Wieland-Gumlich Aldehyde admits two retrosynthetic disconnections: reduction to an ester precursor (recovered by the unguided generator) and deprotection to a TMS-ether precursor (recovered \emph{only} by SCR with rxntype guidance). \textbf{(b)}~A route for Target~49 found exclusively by rxntype guidance; all eight baselines fail to produce any valid route for this target. \textbf{(c)}~A 3-step route for Target~169 found exclusively by Tanimoto guidance toward penicillamine as starting material; the Tanimoto similarity to the SM increases along the route ($0.15 \!\to\! 0.17 \!\to\! 0.60$). No baseline solves this target. Together, the two guidance signals unlock routes for 33 targets (17.4\,\%) inaccessible to all baselines.}
    \label{fig:case-study}
\end{figure*}

\section{Discussion and Limitations}
\label{sec:discussion-and-limitations}
SCR incorporates chemist-specified preferences (reaction class, starting material) at decode time, expanding the reachable set beyond post-hoc filtering, which can only select among candidates the unguided generator already produces. Without a target specification, SCR reduces to the unguided generator. \looseness=-1

Our evaluation depends on forward-model round-trip accuracy and template-based reaction-type classifiers, both of which struggle on novel chemistry.
SCR's ranking objective is generator-dependent--the negative anchor $\hat r_k$ is defined relative to $p_\theta$--so changing the base generator requires retraining.
Tighter coupling with search-level SM filters, multi-objective guidance, and cost-aware route optimization are natural next steps. \looseness=-1

\section{Conclusion}
We introduced SCR, a training procedure that calibrates token-level classifiers for effective guidance in autoregressive retrosynthesis. Our analysis identifies a fundamental limit of cross-entropy training and proves that SCR-trained classifiers expand the chemical space reachable by beam search. SCR steers generation toward desired reaction types and starting materials at inference time, complementing search-level planners. \looseness=-1

\bibliographystyle{plainnat}
\bibliography{references}


\newpage
\appendix
\section*{Appendices}
\startcontents[appendices]

\begingroup
\sc
\printcontents[appendices]{l}{1}{}
\endgroup

\section{Related work}

\paragraph{Steerable synthesis planning.}
Single-step retrosynthesis models span template-based methods
\citep{segler2017neuralsym,chen2021localretro,seidl2021mhnreact}
and template-free approaches
\citep{schwaller2019molecular,Zhong2022Rsmiles,sacha2021megan}.
Reranking \citep{lin2022Reranking} and diversity-aware generation \citep{gaiński2024retrogfn}
improve candidate quality but remain constrained to the base model's output space.
At the multi-step level, search algorithms such as Retro*~\citep{chen2020retrostar},
MCTS~\citep{Segler2018MCTS}, PDVN~\citep{Liu2023PDVN}, and GRASP~\citep{Yu2022GRASP}
orchestrate single-step predictions into full routes.
Constrained planning methods steer search via filtering~\citep{Segler2018MCTS},
value function biasing~\citep{Armstrong2024TangoStar},
starting-material-aware expansion~\citep{Yu2024desp,Tripp2024RetroFallback},
or reinforcement learning~\citep{Guo2024RetrosynthesisZero}.
However, all these search-level interventions operate over candidates already produced
by the underlying single-step model, fundamentally limiting exploration to the chemical
space the base generator can access. Our work intervenes directly at the token generation level,
which is complementary to search-level steering and can expand the reachable chemical space
(\cref{proposition:reachability}).

\paragraph{Classifier guidance in autoregressive models.}
Classifier guidance originated in diffusion models
\citep{dhariwal2021diffusionmodelsbeatgans}
and was extended to classifier-free variants \citep{ho2022classifierfreediffusionguidance}.
In the autoregressive setting, FUDGE \citep{Yang2021fudge} applies Bayesian token-level
reranking based on future attribute prediction;
PPLM \citep{dathathri2019plug} perturbs hidden states via attribute gradients;
GeDi \citep{krause2020gedi} uses a generative discriminator to contrast class-conditional
and unconditional distributions; and NADO \citep{Meng2022Nado} decomposes oracle constraints
into per-token guidance signals.
Several works study cooperative or pragmatic decoding
\citep{holtzman2018cooperative,cohngordon2018pragmatically}.
While these methods demonstrate effective control in natural language settings,
none analyze why standard cross-entropy-trained classifiers fail as guidance signals---even
the CE-optimal classifier achieves a discriminability driven by data statistics rather than
the generator's preferences (\cref{sec:ce-shortcomings}).
Applications to chemical synthesis planning also remain largely unexplored.
Our work addresses both gaps: we provide a theoretical characterization of this failure mode
and a training procedure (SCR) that overcomes it, with formal reachability guarantees for
guided beam search.

\section{Theoretical Analyses and Proofs}
\label{app:theoretical-analysis}

\subsection{Sample-complexity for CE-guidance success}
\label{app:sample-complexity}
 
This section provides the derivation behind Eq.~\eqref{eq:rare-cell-bound},
its verification against direct simulation, and a tighter Binomial-tail
analysis for the regime where the normal approximation strains.
 
\textbf{Setup.} A single prefix $X$, vocabulary $\{a, b\}$, classes
$\{0, 1\}$. The data is parameterized by within-class noise $\epsilon$
(class~0 prefers $a$, class~1 prefers $b$, each with probability
$1-\epsilon$) and class prior $\eta = P(\phi=1 \mid X)$. The generator,
trained on the same distribution, satisfies $p_\theta \approx P_\mathrm{data}$.
Through $N$ joint draws of $(t, \phi)$ we observe a $2 \times 2$ table
of cell counts $n_{t,c}$. The CE classifier's empirical discriminability
at the steering pair $(r^*_k, \hat r_k) = (b, a)$ is
$\hat\Delta_k^{\mathrm{CE}} = \log \hat q_b - \log \hat q_a$ with
$\hat q_t = n_{t,1}/(n_{t,0} + n_{t,1})$, where $\hat q_t$ estimates
$q_t = P(\phi{=}1 \mid X, t)$. Throughout, we work with the
\emph{class-conditional discriminability}
$\Delta_k^{\mathrm{cond}} := \Delta_k^{\mathrm{CE}} - G_k$, the part of
the CE classifier's signal that exceeds the generator's preference and
is therefore actually usable for guidance.

\textbf{Variance via the delta method.}
For a sample proportion $\hat q_t$ estimating $q_t = P(\phi{=}1 \mid X, t)$
from $n_t = n_{t,0} + n_{t,1}$ trials,
$\mathrm{Var}(\hat q_t) = q_t(1-q_t)/n_t$. Applying the delta method to
$\log \hat q_t$ with $f(x) = \log x$, $f'(x) = 1/x$:
\begin{equation}
\mathrm{Var}(\log \hat q_t)
\;\approx\; [f'(q_t)]^2 \,\mathrm{Var}(\hat q_t)
\;=\; \frac{1 - q_t}{n_t \, q_t}
\;\approx\; \frac{1}{\mathbb{E}[n_{t,1}]}\quad\text{when } q_t \ll 1,
\end{equation}
using $n_t \, q_t = \mathbb{E}[n_{t,1}]$ (the expected count of class-1
examples following token $t$). For independent estimators of $\hat q_a$ and $\hat q_b$, the variances add:
\begin{equation}
\mathrm{Var}(\hat\Delta_k^{\mathrm{CE}})
\;\approx\; \frac{1 - q_a}{\mathbb{E}[n_{a,1}]} + \frac{1 - q_b}{\mathbb{E}[n_{b,1}]}.
\label{eq:variance}
\end{equation}
For the symmetric regime $\eta = \epsilon = 0.05$ used as a running example,
the population values give $q_a = \eta\epsilon/[\eta\epsilon + (1-\eta)(1-\epsilon)] \approx 0.003$
and $q_b = \eta(1-\epsilon)/[\eta(1-\epsilon) + (1-\eta)\epsilon] = 0.5$,
together with $\mathbb{E}[n_{a,1}] = N\eta\epsilon$ and $\mathbb{E}[n_{b,1}] = N\eta(1-\epsilon)$.
Per unit $N\eta$, the rare cell contributes
$(1 - q_a)/\epsilon \approx 19.9$ versus
$(1 - q_b)/(1-\epsilon) \approx 0.53$ for the abundant cell --- a dominance
factor of $\sim\!38$. The rare cell governs the variance, and the
approximation $\mathrm{Var}(\hat\Delta_k^{\mathrm{CE}}) \approx 1/\mathbb{E}[n_{a,1}]$
used in Eq.~\eqref{eq:rare-cell-bound} is tight up to a small additive
correction from the abundant cell.

\textbf{From variance to sample complexity.}
The success condition is
$\Pr[\hat\Delta_k^{\mathrm{CE}} > G_k] \ge 1-\delta$. Writing
$\mathrm{Var}_\Delta := \mathrm{Var}(\hat\Delta_k^{\mathrm{CE}})$ for brevity, the normal
approximation $\hat\Delta_k^{\mathrm{CE}} \sim \mathcal{N}(\Delta_k^{\mathrm{CE}}, \mathrm{Var}_\Delta)$ gives:
\begin{equation}
\Pr\!\left[\hat\Delta_k^{\mathrm{CE}} > G_k\right]
\;=\; \Phi\!\left(\frac{\Delta_k^{\mathrm{cond}}}{\sqrt{\mathrm{Var}_\Delta}}\right)
\;\geq\; 1 - \delta
\quad\Longleftrightarrow\quad
\mathrm{Var}_\Delta \;\leq\; \frac{(\Delta_k^{\mathrm{cond}})^2}{z_{1-\delta}^2}.
\label{eq:rare-cell-bound}
\end{equation}
Substituting $\mathrm{Var}_\Delta \approx 1/\mathbb{E}[n_{a,1}]$ yields the
rare-cell sample-complexity bound
$\mathbb{E}[n_{a,1}] \geq z_{1-\delta}^2 / (\Delta_k^{\mathrm{cond}})^2$.

\textbf{Numerical verification.}
Table~\ref{tab:nmin-verification} reports the analytically predicted
$N_{\min}(\eta) = z_{1-\delta}^{2}\!/\!\bigl(\eta\epsilon (\Delta_k^{\mathrm{cond}})^2\bigr)$
together with simulated success rates from $4{,}000$ Monte Carlo trials
at each $(\eta, N_{\min})$ pair. Empirical rates exceed the target
$0.9$ uniformly, confirming the bound is mildly conservative for these
moderate-count regimes.
 
\begin{table}[h]
\centering
\caption{Verification of Eq.~\eqref{eq:rare-cell-bound} at $\epsilon{=}0.05$,
$\delta{=}0.10$. Empirical $\Pr[\hat\Delta_k^{\mathrm{CE}} > G_k]$ is
estimated from $4{,}000$ Monte Carlo trials.}
\label{tab:nmin-verification}
\begin{tabular}{rrrr}
\toprule
$\eta$ & $\mathbb{E}[n_{a,1}]$ at $N_{\min}$ & analytic $N_{\min}$ & empirical $\Pr$ \\
\midrule
0.500 & 0.19 & 8     & 0.999 \\
0.200 & 0.19 & 19    & 0.993 \\
0.100 & 0.19 & 39    & 0.991 \\
0.050 & 0.19 & 78    & 0.985 \\
0.020 & 0.20 & 197   & 0.984 \\
0.010 & 0.20 & 396   & 0.982 \\
\bottomrule
\end{tabular}
\end{table}

\textbf{Beyond the asymptote.}
The normal approximation requires $\mathbb{E}[n_{a,1}] \gtrsim 10$ to be
reliable. When the bound itself returns values below this threshold
(as it does for large $\Delta_k^{\mathrm{cond}}$), the Bernoulli
distribution at the rare cell is closer to Poisson and exhibits a
substantial point mass at zero that the normal misses. The practical
threshold is therefore $\sim 10 \cdot z_{1-\delta}^{2}/(\Delta_k^{\mathrm{cond}})^2$,
giving:
 
\begin{table}[h]
\centering
\caption{Practical rare-cell thresholds at varying class-conditional
margin $\Delta_k^{\mathrm{cond}}$.
At SCR's enforced margin $\gamma = 1$, CE training would require
$\sim 30$ rare-cell examples per (product, prefix) pair to match what
SCR achieves by construction.}
\label{tab:practical-thresholds}
\begin{tabular}{rrrr}
\toprule
$\Delta_k^{\mathrm{cond}}$ & asymp. ($\delta{=}0.05$) & practical ($\times 10$) & USPTO regime \\
\midrule
3.0 & 0.32 & 3.2  & easy (data-implied) \\
2.0 & 0.71 & 7.1  & moderate \\
1.0 & 2.71 & 27.1 & SCR-enforced ($\gamma{=}1$) \\
0.5 & 10.83 & 108 & hard \\
\bottomrule
\end{tabular}
\end{table}

\textbf{Empirical implication.}
Counting rare-cell occurrences in USPTO-50k directly: for each (product,
prefix) pair, count training reactions whose reactant SMILES match the
prefix exactly but diverge at the next-token position into a
non-ground-truth class. The fraction of (product, prefix) pairs reaching
the practical threshold of $\sim 30$ is the directly verifiable
empirical statement. This fraction is negligible on USPTO,
formalizing why CE-trained classifiers are structurally inadequate for
guidance regardless of training-set size.

\subsection{Reaching New Regions of the Chemical Space}
\label{app:chemical-space-exploration}


\begin{remark}[Classifier bounds and SCR training]
\label{rem:margin-to-bounds}
\Cref{proposition:reachability} requires absolute bounds $c_1, c_2$ on the
classifier posterior, with the upper bound active only at divergent
prefixes. The ranking margin $\gamma$ alone does not deliver such absolute
bounds: a guided-score gap $\tilde s_\alpha(r^*_k) - \tilde s_\alpha(\hat r_k) \geq \gamma$
constrains a difference between two log-scores, fixing a softmax \emph{ratio}
but not pinning either probability---both can approach zero with the gap
preserved. The two components of the SCR objective combine to supply the
two bounds in \eqref{eq:classifier-bounds}.

\textbf{CE component (absolute calibration $\to c_1$).}
Cross-entropy training on partial sequences pulls
$p_\alpha(\phi_\alpha = 1 \mid \rvec_{\leq l}, p)$ toward the empirical class
posterior, yielding the lower bound $c_1$ on property-satisfying prefixes.
We treat $c_1$ as an empirically verified quantity rather than a theoretical
guarantee: on held-out near-complete USPTO-50k partial sequences (completion
ratio $\geq 0.95$), the SCR classifier assigns mean probability $\approx 0.86$
to the correct class (top-1 accuracy $\approx 92.6\%$ over the 12-way label
space of 11 reaction types plus the contrastive wrong-token class;
\cref{sec:training-property-predictors-experiments}),
giving $c_1 \approx 0.85$.

\textbf{Ranking component (gap at divergent positions $\to c_2$).}
The negative anchor in the SCR ranking loss is constructed exactly at
divergent positions $r^*_k \neq \hat r_k$, where it enforces a logit gap of
$\gamma$. Combined with the $c_1$ floor above, this yields the upper bound
$c_2 < c_1$---supplying \eqref{eq:classifier-bounds} at precisely the
positions where the proof requires it.

The threshold $\lambda^*$ scales inversely with the number of divergent
positions $|D(\rvec, l)|$ between $\rvec^*$ and competing beam sequences. On
retrosynthesis SMILES this is typically $5$--$20$ of $\sim$$50$ tokens
(\cref{sec:ce-shortcomings}), keeping $\lambda^*$ at moderate values in
practice.
\end{remark}

\begin{proposition*}[Reachability under Guided Beam Search, restated]
Let $\beta_n^{(l)}$ be the score of the $n$-th beam (i.e. the beam with
the lowest accepted score) at generation step $l$, and assume
$\mathcal{R}_\phi \setminus \mathrm{Beam}_\omega \neq \emptyset$.
Let $\rvec^* \in \mathcal{R}_\phi \setminus \mathrm{Beam}_\omega$ be any
property-satisfying sequence absent from the unguided beam.
Assume the classifier $p_{\alpha}$ satisfies for all $l \in [1, L]$:
\begin{align}
\label{eq:classifier-bounds}
    p_{\alpha}(\phi_\alpha &= 1 \mid \rvec^*_{\leq l}, p) \geq c_1 \\
    p_{\alpha}(\phi_\alpha &= 1 \mid \rvec_{\leq l}, p) \leq c_2,
    \quad \text{when } \phi_{\alpha}(\rvec, p) = 0 \text{ and } \rvec_{\leq l} \neq \rvec^*_{\leq l}
\end{align}
where $0 < c_2 < c_1 \leq 1$.
The bound on non-property-satisfying sequences is required only at
prefixes that diverge from $\rvec^*$; at shared prefixes, the
classifier output is unconstrained, since both sequences receive the
same per-step classifier contribution and these cancel in any beam
comparison. Then:
\begin{enumerate}
    \item Unguided beam search ($\lambda = 0$) excludes $\rvec^*$ from
    the final output,
    \item There exists $\lambda^* > 0$ such that guided beam search
    with guidance scale $\lambda \geq \lambda^*$ includes at least one
    property-satisfying sequence in the final output.
\end{enumerate}
\end{proposition*}
\begin{proof}
We define cumulative scores at generation step $l$ as
\begin{align}
F(\rvec_{\leq l}) &= \sum_{t=1}^{l} \log p_\theta(r_t \mid \rvec_{<t}, p), \\
F_{\text{guided}}(\rvec_{\leq l}) &= F(\rvec_{\leq l}) + \lambda \sum_{t=1}^{l} \log p_\alpha(\phi_\alpha = 1 \mid \rvec_{\leq t}, p).
\end{align}

\textbf{Statement 1} follows directly from $\rvec^* \in \mathcal{R}_\phi \setminus \mathrm{Beam}_\omega$:
$\rvec^*$ is by definition absent from the unguided beam, hence excluded from
the unguided final output.

\textbf{Statement 2.} We show by induction on $l \in [1, L]$ that, for
sufficiently large $\lambda$, $\rvec^*_{\leq l} \in B^g_l$, where $B^g_l$ denotes
the guided beam at step $l$. Since $\rvec^* \in \mathcal{R}_\phi$, this gives
$B^g_L \cap \mathcal{R}_\phi \neq \emptyset$, establishing the conclusion.

For any competitor $\rvec$ with $\rvec_{\leq l} \neq \rvec^*_{\leq l}$, let
\begin{equation}
D(\rvec, l) \;:=\; \{\, t \in [1, l] : \rvec_{\leq t} \neq \rvec^*_{\leq t} \,\}
\end{equation}
denote its divergent positions up to step $l$; since $\rvec_{\leq l} \neq \rvec^*_{\leq l}$,
$|D(\rvec, l)| \geq 1$. Define the per-step threshold
\begin{equation}
\lambda^*_l \;:=\; \max\!\left(0,\;
\max_{\substack{\rvec \notin \mathcal{R}_\phi \\ \rvec_{\leq l} \neq \rvec^*_{\leq l}}}
\frac{F(\rvec_{\leq l}) - F(\rvec^*_{\leq l})}{|D(\rvec, l)| \cdot \log(c_1/c_2)}\right),
\qquad \lambda^* := \max_{l \in [1, L]} \lambda^*_l.
\end{equation}
The denominator is positive since $c_1 > c_2 > 0$.

\textbf{Inductive step.} Fix $\lambda \geq \lambda^*$. Assume
$\rvec^*_{\leq l-1} \in B^g_{l-1}$ (vacuously true at $l = 1$). Then
$\rvec^*_{\leq l}$ is a candidate at step $l$. Suppose for contradiction that
$\rvec^*_{\leq l} \notin B^g_l$. Then $B^g_l$ consists of $\omega$ candidates each
with $F_{\text{guided}}(\cdot) > F_{\text{guided}}(\rvec^*_{\leq l})$, and each
satisfies $\rvec_{\leq l} \neq \rvec^*_{\leq l}$.

If \emph{any} member of $B^g_l$ is the prefix of a sequence in $\mathcal{R}_\phi$
that survives to $B^g_L$, we are done: the final beam contains a
property-satisfying sequence, and Statement 2 holds without recovering
$\rvec^*$ specifically. Otherwise, every $\rvec \in B^g_l$ corresponds to a
non-property-satisfying full sequence. Decomposing $[1, l]$ into shared and
divergent positions, the per-step classifier contributions cancel at shared
positions ($\rvec_{\leq t} = \rvec^*_{\leq t}$ implies identical
$p_\alpha$-output), leaving
\begin{align*}
F_{\text{guided}}(\rvec_{\leq l}) - F_{\text{guided}}(\rvec^*_{\leq l})
&= \bigl[F(\rvec_{\leq l}) - F(\rvec^*_{\leq l})\bigr]
+ \lambda \!\!\sum_{t \in D(\rvec, l)}\!\! \bigl[\log p_\alpha(\phi_\alpha = 1 \mid \rvec_{\leq t}, p) - \log p_\alpha(\phi_\alpha = 1 \mid \rvec^*_{\leq t}, p)\bigr] \\
&\leq \bigl[F(\rvec_{\leq l}) - F(\rvec^*_{\leq l})\bigr] - \lambda \, |D(\rvec, l)| \log(c_1/c_2),
\end{align*}
where the inequality uses $\log p_\alpha(\phi_\alpha = 1 \mid \rvec^*_{\leq t}, p) \geq \log c_1$ and
$\log p_\alpha(\phi_\alpha = 1 \mid \rvec_{\leq t}, p) \leq \log c_2$ at $t \in D(\rvec, l)$,
both supplied by \eqref{eq:classifier-bounds} (the upper bound applies precisely because
$\rvec_{\leq t} \neq \rvec^*_{\leq t}$ at divergent positions). Combining with
$\lambda \geq \lambda^*_l$ gives
$F_{\text{guided}}(\rvec_{\leq l}) - F_{\text{guided}}(\rvec^*_{\leq l}) \leq 0$,
contradicting the assumption that every $\rvec \in B^g_l$ strictly outranks
$\rvec^*_{\leq l}$. Hence $\rvec^*_{\leq l} \in B^g_l$, completing the induction.

Setting $l = L$ gives $\rvec^* \in B^g_L$, and since $\rvec^* \in \mathcal{R}_\phi$,
the final guided output contains a property-satisfying sequence.
\end{proof}

\section{Guided Properties}
\label{app:properties-additional}
We guide generation toward two properties of particular relevance to synthetic chemists:
reaction type and starting material similarity.
Both arise naturally in synthesis planning and complement each other---reaction type
controls the disconnection strategy, while starting material similarity steers routes
toward accessible building blocks.

\subsection{Reaction type}
Chemists often prefer specific reaction mechanisms based on their expertise, laboratory
capabilities, and available equipment~\citep{nicolaou1996classics}.
Early computational work showed that conditioning on reaction type substantially improves
retrosynthesis accuracy, with template-based methods gaining up to 11\% in top-10
predictions when the reaction class is known a priori~\citep{segler2017neuralsym}.
We adopt the 12 superclass taxonomy of \citet{schneider-uspto}, available directly for
USPTO-50k; for USPTO-190, we assign classes using RXN-Insight~\citep{rxn_insight}.

\subsection{Starting material similarity}
Practical synthesis routes must ultimately connect to commercially available or
laboratory-accessible building blocks~\citep{Yu2024desp,Armstrong2024TangoStar}.
We operationalize this constraint by guiding toward high Tanimoto similarity
(Morgan fingerprints) between the predicted reactants and a designated starting material,
encouraging the model to propose routes through specified building blocks without
requiring exact matches at the single-step level.

\section{Implementation and Experimental Details}
\label{app:datasets-and-processing-additional}
We present here the datasets, models, training and inference procedures, baselines, and
evaluation pipeline used in this work. For the new datasets introduced, we explain the
generation procedure, intended purpose, and key statistics.

\subsection{Reaction and route data}
\label{app:reaction-route-data}
We rely on two publicly available reaction datasets and derive a third for single-step evaluation.

\paragraph{USPTO-50k.}
The USPTO-50k dataset \citep{schneider-uspto} contains approximately 50{,}000 reactions
extracted from the United States Patent and Trademark Office (USPTO) filings. Each reaction is
annotated with one of 10 reaction superclasses (plus a miscellaneous class) defined by
\citet{schneider-uspto} based on functional group transformations. We use the standard
train/validation/test splits and train both single-step models and reaction type classifiers on this data.

\paragraph{USPTO-190 (USPTO-Hard).}
USPTO-190, also known as USPTO-Hard, was introduced by \citet{chen2020retrostar} as a
challenging multi-step benchmark. It consists of 190 target molecules whose synthesis routes
were curated from USPTO patent literature. Each target has a ground-truth synthesis route
comprising multiple reaction steps. Unlike USPTO-50k, this dataset does not include reaction type annotations;
we use RXN-Insight \citep{rxn_insight} to assign Schneider reaction classes to the individual steps.

\paragraph{USPTO-190-Steps (new).}
To evaluate single-step guidance on a diverse set of products, we extract the individual reactions
composing the test routes of USPTO-190 and deduplicate them by product SMILES. This yields 640 unique
products, each associated with its ground-truth precursors and reaction type. We refer to this dataset
as USPTO-190-Steps throughout the paper. The distribution of reaction types in this dataset reflects the
composition of the original multi-step routes, with Heteroatom Alkylation (class~2),
Acylation (class~9), and C--C Coupling (class~0) being the most common classes.

\subsection{Reaction type partial-sequence data}
\label{app:reaction-type-data}
We generate two datasets for training the reaction classifiers: one from USPTO-50k
and one from USPTO-190. The procedure in both cases is the same. The procedure described below is applied to the training and validation
subsets of both USPTO-190-steps and USPTO-50k.
\begin{itemize}
\item We augment each product with 3 root-aligned canoncalizations, obtained by first generating 20 canonicalizations then selecting the
3 most different ones.
\item (Sample): We then sample for each product smiles (including the root-aligned augmentations) 100 samples using the generator.
\item We then split the samples dataset by reaction type, thus obtaining precursors for each reaction type from the products available in the dataset
\item (Contrast): Next we generate the partial sequences by splitting each precursor smiles at every position, and augmenting with contrastive samples
by swapping out the last token with 3 random tokens sampled from the vocabulary.
\item Lastly, we subsample the partial sequences dataset to obtain a training set of 2M sequences and a validation set of 200K sequences covering different
completion ratios.
\end{itemize}
\begin{table}[htb!]
\caption{Class-wise statistics of the reaction-type partial-sequence dataset built from USPTO-50k: number of original reactions, generator samples, and the (oversampled) per-class subsample used for training. The training and validation sets additionally contain 2.5M and 1.0M contrastive wrong-token sequences (class~11), giving totals of 4.84M training and 1.77M validation partial sequences.}
\label{tab:reaction-type-50k-stats}
\centering
\small
\begin{tabular}{lccccccccccc}
\toprule
    Dataset & 0 & 1 & 2 & 3 & 4 & 5 & 6 & 7 & 8 & 9 & 10 \\
\midrule
    \multicolumn{12}{l}{\textit{uspto50k-train}} \\[0pt]
    \quad original & 12.4k & 8.9k & 4.3k & 632 & 1.0k & 4.0k & 6.3k & 632 & 106 & 649 & 853 \\
    \quad samples & 472.9k & 241.6k & 347.5k & 336.3k & 75.5k & 69.1k & 267.0k & 63.4k & 34.0k & 329.9k & 310.7k \\
    \quad subsamples & 227k & 227k & 227k & 227k & 227k & 227k & 227k & 227k & 68k & 227k & 227k \\
    \multicolumn{12}{l}{\textit{uspto50k-val}} \\[0pt]
    \quad original & 1.5k & 1.1k & 538 & 75 & 127 & 487 & 789 & 71 & 15 & 90 & 113 \\
    \quad samples & 256.6k & 163.1k & 141.1k & 129.6k & 17.9k & 31.4k & 122.4k & 18.1k & 12.6k & 132.5k & 116.8k \\
    \quad subsamples & 91k & 91k & 91k & 53k & 81k & 91k & 91k & 48k & 14k & 53k & 70k \\
\bottomrule
\end{tabular}
\end{table}

\subsection{Tanimoto similarity partial-sequence data}
\label{app:tanimoto-data}
For the starting material guidance task, we train a tanimoto similarity predictor that estimates the
Morgan fingerprint Tanimoto similarity between the completed reactants and a designated starting material.
The predictor follows the same SCR framework as the reaction type classifier: partial sequences are generated
from ground-truth and on-policy samples, augmented with contrastive wrong-token examples,
and trained with the combined cross-entropy (adapted for regression) and ranking loss objectives,
but the classification head is replaced with an MSE regression head.
The training data is constructed analogously to the reaction type data (\cref{app:reaction-type-data}),
with target labels being the Tanimoto similarity between the completed reactant set and the designated starting material
in each route.

\subsection{Property predictor training}
\label{app:experimental-setup}
All property classifiers are vanilla transformer encoders with 4~layers, 4~attention heads,
128 hidden dimensions, and 0.1~dropout. We train with Adam (learning rate $10^{-3}$, weight decay $10^{-5}$)
on 4.84M partial sequences for the reaction type task (\cref{tab:reaction-type-50k-stats}) and an equivalently
constructed dataset for the tanimoto similarity task.
The ranking loss margin is $\gamma = 1.0$ and the ranking weight is $\mu = 1.0$ unless stated otherwise.
Each classifier is validated on a held-out set of 1.77M partial sequences generated from ground-truth
precursors and round-trip validated generator samples, augmented with contrastive wrong-token completions.
The full training procedure is summarized in Algorithm~\ref{alg:app-classifier-training}.

\begin{algorithm}[t]
\caption{Training Token-Level Property Classifier}
\label{alg:app-classifier-training}
\begin{algorithmic}[1]
\REQUIRE Ground-truth dataset $\mathcal{D} = \{(\rvec^{(i)}, p^{(i)}, y^{(i)})\}_{i=1}^N$, generator $p_\theta$, vocabulary $\mathcal{V}$
\REQUIRE Hyperparameters: margin $\gamma$, ranking weight $\mu$, sampling ratio $\rho$
\ENSURE Trained classifier $p_\alpha$
\STATE Initialize classifier $p_\alpha$
\FOR{each training iteration}
    \STATE Sample minibatch $\{(\rvec, p, y)\}$ from $\mathcal{D}$
    \STATE $\mathcal{B} \gets \emptyset$ \MYCOMMENT{Training batch}
    \FOR{each $(\rvec, p, y)$ in minibatch}
        \STATE Sample position $k \sim \text{Uniform}(1, |\rvec|)$
        \STATE $\mathcal{B} \gets \mathcal{B} \cup \{(\rvec_{\leq k}, p, y)\}$ \MYCOMMENT{Ground-truth partial sequence}
        \STATE Sample $\tilde{r}_k \sim \text{Uniform}(\mathcal{V})$
        \STATE $\mathcal{B} \gets \mathcal{B} \cup \{((\rvec_{<k}, \tilde{r}_k), p, y_{\text{wrong}})\}$ \MYCOMMENT{Wrong-token augmentation}
        \IF{$u < \rho$ where $u \sim \text{Uniform}(0,1)$}
            \STATE $\tilde{\rvec} \sim p_\theta(\cdot \mid p)$ \MYCOMMENT{Sample from generator}
            \STATE $\tilde{y} \gets \textsc{GetReactionClass}(\tilde{\rvec}, p)$
            \STATE $\mathcal{B} \gets \mathcal{B} \cup \{(\tilde{\rvec}_{\leq k}, p, \tilde{y})\}$ \MYCOMMENT{On-policy sample}
        \ENDIF
    \ENDFOR
    \STATE $\mathcal{L}_{\text{CE}} \gets \frac{1}{|\mathcal{B}|}\sum_{(\rvec_{\leq k}, p, y) \in \mathcal{B}} -\log p_\alpha(y \mid \rvec_{\leq k}, p)$
    \STATE Let $\mathcal{B}' \subseteq \mathcal{B}$ denote the ground-truth partial sequences, with $r^*_k := \rvec_k$ \MYCOMMENT{Excludes wrong-token and on-policy samples}
    \STATE $\hat{r}_k \gets \argmax_{r \neq r^*_k} p_\theta(r \mid \rvec_{<k}, p)$ \MYCOMMENT{Generator's preferred alternative to ground truth}
    \STATE $\tilde s_\alpha(r) \gets \log p_\theta(r \mid \rvec_{<k}, p) + \log p_\alpha(y \mid \rvec_{<k}, r, p)$ \MYCOMMENT{Guided log-score (generator + classifier), log-softmax-normalized over $\{r^*_k\} \cup \text{top-}k(p_\theta)$}
    \STATE $\mathcal{L}_{\text{rank}} \gets \frac{1}{|\mathcal{B}'|}\sum_{(\rvec_{\leq k}, p, y) \in \mathcal{B}'} \max\big(0,\, \gamma + \tilde s_\alpha(\hat{r}_k) - \tilde s_\alpha(r^*_k)\big)$
    \STATE Update $p_\alpha$ to minimize $\mathcal{L}_{\text{CE}} + \mu \mathcal{L}_{\text{rank}}$
\ENDFOR
\STATE \textbf{return} $p_\alpha$
\end{algorithmic}
\end{algorithm}

\subsection{Inference}
\label{app:lookahead-inference}

\paragraph{Base generator.}
We use the product-reactant-aligned generator described in \cref{sec:single-step-retro} with a beam size of~10.
To obtain 100 samples per product despite a beam width of~10, we follow the RSMILES protocol \citep{Zhong2022Rsmiles}---which is also the syntheseus default for template-free models---and augment each input product with 20 canonical-equivalent SMILES randomizations. Beam search is run independently per augmentation, and the resulting candidates are pooled and deduplicated to produce the per-product sample set (the same procedure used to construct training data, \cref{app:reaction-type-data}).
We evaluate two checkpoints: one trained on USPTO-50k and one on USPTO-190,
to study how guidance interacts with generators of different training-set coverage.

\paragraph{Default inference parameters.}
Guided generation uses guidance scale $\lambda = 1.0$, guidance onset length $L_{\min} = 5$,
and $N = 72$ candidate tokens per step, unless stated otherwise.

\paragraph{Lookahead variant.}
We compare three inference strategies for applying guidance during search:
(i)~\emph{no lookahead} (SCR-nla), which applies guidance at a fixed scale
$\lambda$ throughout generation; (ii)~\emph{lookahead} (SCR-la), which uses
Algorithm~\ref{alg:lookahead} to adaptively select $\lambda$ per target from a
candidate set; and (iii)~\emph{lookahead-fast} (SCR-la-fast), which performs
exhaustive parameter search over all candidates.
All variants use reaction-type guidance on USPTO-190 with 100 model calls.
By default we apply no lookahead; when the lookahead variant is used we denote it as SCR-la.

\begin{algorithm}[ht]
    \caption{Lookahead Parameter Selection for Guided Generation}
    \label{alg:lookahead}
    \begin{algorithmic}[1]
    \REQUIRE Product molecule $p$, total sample budget $N$, candidate parameters $\Lambda = \{\lambda_1, \ldots, \lambda_m\}$, exploration samples per parameter $n_{\text{explore}}$
    \ENSURE Set of $N$ generated precursors with optimal guidance
    \STATE $n_{\text{exploit}} \gets N - m \cdot n_{\text{explore}}$ \COMMENT{Reserve samples for exploitation}
    \STATE $\mathcal{P} \gets \emptyset$ \COMMENT{Initialize precursor set}

    \FOR{$\lambda \in \Lambda$}
        \STATE Generate $n_{\text{explore}}$ precursors: $\{r_1^\lambda, \ldots, r_{n_{\text{explore}}}^\lambda\} \sim p_\theta^\lambda(\cdot | p, \phi_\alpha=1)$
        \STATE Compute property scores: $s_i^\lambda \gets \phi_\alpha(r_i^\lambda, p)$ for $i = 1, \ldots, n_{\text{explore}}$
        \STATE $\bar{s}^\lambda \gets \frac{1}{n_{\text{explore}}} \sum_{i=1}^{n_{\text{explore}}} s_i^\lambda$ \COMMENT{Mean property satisfaction}
        \STATE $\mathcal{P} \gets \mathcal{P} \cup \{r_1^\lambda, \ldots, r_{n_{\text{explore}}}^\lambda\}$
    \ENDFOR

    \STATE $\lambda^* \gets \argmax_{\lambda \in \Lambda} \bar{s}^\lambda$ \COMMENT{Select best parameter}

    \STATE Generate $n_{\text{exploit}}$ additional precursors with $\lambda^*$: $\{r_1^*, \ldots, r_{n_{\text{exploit}}}^*\} \sim p_\theta^{\lambda^*}(\cdot | p, \phi_\alpha=1)$ \COMMENT{Exploitation phase}
    \STATE $\mathcal{P} \gets \mathcal{P} \cup \{r_1^*, \ldots, r_{n_{\text{exploit}}}^*\}$

    \STATE \textbf{return} $\mathcal{P}$
    \end{algorithmic}
\end{algorithm}

\subsection{Baselines}
For single-step evaluation, we compare against all baselines available in the
Syntheseus library \citep{maziarz2023syntheseus}, including both template-free methods
(MEGAN \citep{sacha2021megan}, Chemformer, Graph2Edits) and template-based methods
(LocalRetro \citep{chen2021localretro}, RetroKNN, GLN, MHNReact \citep{seidl2021mhnreact}).
We additionally include NeuralSym \citep{segler2017neuralsym}, a template-based method
commonly used in multi-step planning.
All methods generate 100 samples per product to ensure an equal computational budget.
Template-based baselines are marked with~$^\mathcal{T}$ in result tables.

\subsection{Evaluation pipeline}
\label{app:evaluation}

\paragraph{Chemical validity.}
Generated SMILES strings are validated using Syntheseus \citep{maziarz2023syntheseus},
which checks for parsability and chemical validity via RDKit.
Invalid SMILES are discarded before computing any downstream metrics.

\paragraph{Reaction type classification.}
We use RXN-Insight \citep{rxn_insight} to classify each valid generated reaction
into one of the Schneider reaction superclasses \citep{schneider-uspto}.
A sample is considered \emph{class-correct} if the predicted reaction type matches
the guidance target.

\paragraph{Round-trip accuracy.}
To assess whether a generated precursor set plausibly produces the target product,
we apply the forward model of \citet{Zhong2022Rsmiles} to the generated
reactants and check whether the predicted product matches the original target
(exact SMILES match after canonicalization). This round-trip check provides a
proxy for chemical feasibility without requiring wet-lab validation.

\paragraph{Steering breadth.}
For each product in the dataset, we guide generation toward every available
reaction type and record whether at least one sample is class-correct.
The \emph{guidance success} is the count of products
achieving this for a given target class. The \emph{steering breadth} is defined
as the average number of reaction types for which guidance succeeds per product,
normalized by the total number of classes.
We also report stricter variants that additionally require RXN-Insight reaction
identification (+RXN-I) or round-trip validity (+RT), measuring whether steered
samples are not only class-correct but also chemically identified or feasible.

\paragraph{Jaccard overlap.}
To quantify whether guidance explores genuinely different chemical space,
we compute the Jaccard similarity between the set of unique valid SMILES
produced by guided and unguided generation for each product, then report
the average across all products.

\subsection{Compute resources}
\label{app:compute}
All experiments were run on the LUMI supercomputer using a single AMD Instinct
MI250X GCD per job (\texttt{small-g} partition, 100\,GB RAM); no multi-GPU
training or inference was used.

\paragraph{Property predictor training.}
The reaction-type and Tanimoto-similarity classifiers are small Transformers
(${\sim}5$\,MB; 4 layers, hidden dim 128, 4 heads) trained on partial sequences
sampled from the base generator. On USPTO-Full (1.36\,M training prefixes), one
training epoch takes approximately 1030\,s; the reported reaction-type
checkpoint converges by epoch 35, for roughly $10$ GPU-hours of pure training
time. On USPTO-50K, an epoch takes approximately 770\,s and the reported
checkpoint converges by epoch 45, for roughly $10$--$12$ GPU-hours. Tanimoto
predictors train at the same per-epoch cost on each dataset.

\paragraph{Single-step inference.}
With the R-SMILES generator at beam size 10 and 20 augmentations
(100 candidates per target), the unguided baseline takes
approximately 10--15\,s per target on a single MI250X.
SCR-guided inference (reaction type or Tanimoto, $\lambda{=}1.0$,
including round-trip evaluation) takes approximately 50--60\,s per target,
roughly a $4$--$5{\times}$ overhead over the unguided baseline due to per-token
classifier scoring of the $n_{\mathrm{candidates}}$ continuations.

\paragraph{Multi-step search.}
Retro$^{*}$ search with a 100-call model budget runs in approximately 600--700\,s
of wall-clock per target with SCR-nla guidance (cf.\ the 643.6\,s vs.\ 734.3\,s
comparison in \cref{sec:multi-step}); search jobs were dispatched in
SLURM array tasks of 5 targets each.

\section{Additional Results}
\label{app:additional-results}
\subsection{Single-step synthesis}
\label{app:additional-results-single-step}

\begin{figure*}[htbp]
    \centering
    \resizebox{0.60\textheight}{!}{
        \begin{subfigure}[t]{0.48\textwidth}
        \centering
        \vspace{0pt}
        \textbf{(1) Baseline failure rates}
        \vspace{0.5em}
        \begin{tabular}{lcc}
        \toprule
        & TG & Generator \\
        \midrule
        Failure rate & - & 287 \\
        Exact recovery & 113 & - \\
        RT recovery & 231 & - \\
        \bottomrule
        \end{tabular}
        \vspace{1.5em}
        \textbf{(2) Detailed guidance effects}
        \vspace{0.5em}
        \begin{tabular}{lcc}
        \toprule
        & TG & Generator \\
        \midrule
        Avg. unique samples & \textbf{46.4} & 30.5 \\
        Avg. class correct & \textbf{13.4} & 8.9 \\
        Avg. name correct & \textbf{19.1} & 15.7 \\
        Avg. RT correct & 18.6 & \textbf{18.9} \\
        \bottomrule
        \end{tabular}
        \caption{Effect of guidance on generator samples}
        \label{fig:failure}
        \end{subfigure}
        \hfill
        \begin{subfigure}[t]{0.48\textwidth}
        \centering
        \vspace{0pt}
        \includegraphics[width=\textwidth]{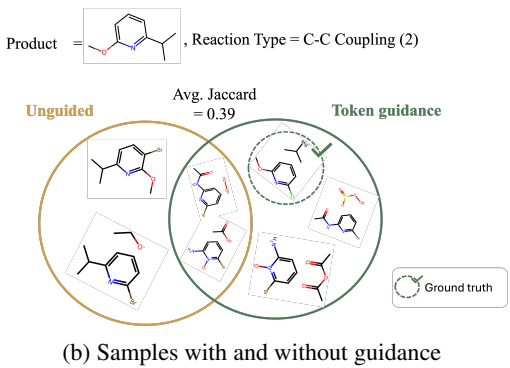}
        \caption{Samples with and without guidance}
        \label{fig:jaccard}
        \end{subfigure}
}
\caption{Effect of guidance on the samples of the generator. The tables show metrics computed over 100 samples per product,
    while the image visualizes samples for a single product.}
\label{fig:effect-of-guidance-on-samples}
\end{figure*}

\paragraph{Ablation on the training paradigm}
To validate the contribution of each SCR component, we train classifiers with different combinations
of training data source (ground-truth vs.\ on-policy generator samples), contrastive augmentation
(with or without wrong-token examples), and ranking loss (with various margin and weight settings).
We then use each classifier to guide single-step generation on USPTO-190-Steps (640 products, 100 samples per product)
and compare route recovery, top-$k$ accuracy, class accuracy, and round-trip accuracy.
SCR with on-policy samples and ranking loss consistently outperforms all other configurations,
confirming that all three components---on-policy training, contrastive augmentation, and margin-based ranking---are
necessary for effective token-level guidance. The detailed results are consistent with the validation accuracy
trends shown in \cref{fig:src-ablation}.

\paragraph{Guiding towards ground truth data}
As an upper-bound analysis, we guide each product toward its ground-truth reaction type and measure
how many complete routes can be recovered---i.e., whether the ground-truth reactant appears among
the precursors returned by the single-step model at all.
\Cref{tab:ground-truth-reaction-type-topk} reports full route recovery, top-$k$ accuracy, and the
average number of class-correct samples per product, broken down by reaction class.

\begin{table}[htb!]
\caption{Class analysis with ground-truth reaction types on USPTO-190 (n\_products=640, samples=100).We report the number of routes recovered completely, in addition to top-1, top-5, top-50, and the average number of samples that are of the correct type per product.}
\label{tab:ground-truth-reaction-type-topk}
\centering
\small
\setlength{\tabcolsep}{4pt}
\begin{tabular}{lcccccccccccc}
    \toprule
    Method & Route & Top1 & Top5 & Top50 &  Cls=0  &  Cls=1  &  Cls=2  &  Cls=3  &  Cls=7  &  Cls=8  &  Cls=9  &  Cls=10  \\
     & ($\uparrow$) & ($\uparrow$) & ($\uparrow$) & ($\uparrow$) & ($\uparrow$) & ($\uparrow$) & ($\uparrow$) & ($\uparrow$) & ($\uparrow$) & ($\uparrow$) & ($\uparrow$) & ($\uparrow$) \\
    \midrule
    \multicolumn{13}{l}{\textit{Trained on USPTO-50k}} \\[2pt]
    \quad Localretro$^\mathcal{T}$ & \underline{7.00} & \underline{0.33} & 0.48 & \underline{0.52} & 18.29 & 13.56 & \underline{22.01} & 4.42 & 2.05 & 0.25 & 9.36 & 3.07 \\
    \quad GLN$^\mathcal{T}$ & 1.00 & 0.29 & 0.45 & 0.48 & 16.65 & 8.29 & 12.59 & 0.23 & 0.89 & 0.25 & \textbf{15.36} & 3.13 \\
    \quad RetroKNN$^\mathcal{T}$ & \underline{7.00} & \underline{0.33} & 0.48 & 0.51 & 19.49 & 13.45 & 20.13 & 4.65 & 1.84 & 0.25 & 9.51 & 3.47 \\
    \quad MHNReact$^\mathcal{T}$ & 3.00 & 0.29 & 0.45 & 0.48 & 15.75 & 6.47 & 10.53 & 0.23 & 0.95 & 0.25 & 8.08 & 2.33 \\
    \quad Megan & 2.00 & 0.28 & 0.43 & 0.49 & \textbf{21.03} & \underline{16.63} & \textbf{22.91} & 2.32 & \underline{4.53} & 1.00 & 8.72 & \underline{8.00} \\
    \quad Graph2Edits & 1.00 & 0.32 & 0.47 & 0.49 & 15.85 & 6.32 & 17.12 & \textbf{16.90} & 3.05 & 0.25 & 8.05 & 2.44 \\
    \quad Chemformer & 2.00 & 0.30 & 0.38 & 0.38 & 2.96 & 1.67 & 3.86 & 3.13 & 1.00 & 0.00 & 2.13 & 2.00 \\
    \quad Generator & 5.00 & \underline{0.33} & \underline{0.49} & 0.51 & 15.04 & 12.52 & 12.13 & 12.48 & 3.53 & \underline{2.50} & 6.49 & 5.96 \\
    \rowcolor{highlightgreen}
    \quad SCR$_{\text{rxn}}$ & \textbf{9.00} & 0.32 & \underline{0.49} & \textbf{0.53} & \underline{20.89} & \textbf{19.60} & 19.52 & 14.71 & \textbf{4.89} & \textbf{7.00} & \underline{10.49} & \textbf{8.47} \\
    \rowcolor{highlightgreen}
    \quad SCR$_{\text{tan}}$ & 6.00 & \textbf{0.35} & \textbf{0.51} & \textbf{0.53} & 17.28 & 14.07 & 14.29 & \underline{14.77} & 3.68 & 2.00 & 6.90 & 6.00 \\[4pt]
    \multicolumn{13}{l}{\textit{Trained on USPTO-190}} \\[2pt]
    \quad NeuralSym$^\mathcal{T}$ & 31.00 & 0.48 & 0.65 & 0.69 & 8.88 & 3.39 & 5.47 & 1.19 & 1.95 & 1.25 & 6.05 & 3.13 \\
    \quad Generator & 54.00 & \textbf{0.59} & \underline{0.76} & 0.81 & 26.89 & 23.43 & 23.80 & 37.65 & 6.89 & 5.00 & 12.92 & 11.40 \\
    \rowcolor{highlightgreen}
    \quad SCR$_{\text{rxn}}$ & \textbf{75.00} & 0.55 & \underline{0.76} & \underline{0.82} & \textbf{44.35} & \textbf{38.80} & \textbf{40.34} & \textbf{55.81} & \textbf{12.74} & \textbf{13.75} & \textbf{22.46} & \textbf{25.11} \\
    \rowcolor{highlightgreen}
    \quad SCR$_{\text{tan}}$ & \underline{73.00} & \underline{0.58} & \textbf{0.84} & \textbf{0.88} & \underline{34.59} & \underline{30.44} & \underline{29.10} & \underline{39.45} & \underline{7.84} & \underline{9.25} & \underline{18.08} & \underline{14.69} \\
    \bottomrule
\end{tabular}
\end{table}

With a USPTO-50k checkpoint, SCR$_{\text{rxn}}$ recovers 9 complete routes compared to 7 for
the best unguided baseline (RetroKNN$^\mathcal{T}$), an improvement of 2 additional routes.
The gains are most pronounced for rare reaction types where the base model struggles:
for class~8, SCR$_{\text{rxn}}$ produces 7.00 class-correct samples per product versus
2.50 for the unguided generator, and for class~10 the counts rise from 5.96 to 8.47.
With a USPTO-190 checkpoint, the gap widens further: SCR$_{\text{rxn}}$ recovers 75 routes
versus 54 for the unguided generator---a 39\% improvement---with per-class sample counts increasing
consistently across all reaction types, and the largest absolute gains again appearing
on underrepresented classes (e.g., class~8: 13.75 vs.\ 5.00; class~10: 25.11 vs.\ 11.40).
Tanimoto guidance (SCR$_{\text{tan}}$) also improves route recovery
(6 routes on USPTO-50k, 73 on USPTO-190) despite not directly targeting reaction type,
suggesting that starting-material-aware guidance indirectly diversifies the reaction
mechanisms explored by the model.

\paragraph{Guiding towards all reaction type assignment pairs}
To evaluate steering beyond the ground-truth class, we guide each product toward every
non-ground-truth reaction type and compare the number of class-correct products against
the unguided baseline. \Cref{fig:delta-heatmaps} shows the difference in success counts,
stratified by source class (row) and target class (column).
Guidance improves steering rates for the vast majority of class pairs: most cells are
positive, indicating that the guided model produces more class-correct samples than
the unguided baseline for the corresponding (source, target) combination.
The improvements are particularly strong for off-diagonal pairs where the target class
differs substantially from the source class, confirming that guidance steers the model
into genuinely different regions of the chemical space rather than reinforcing existing
biases.
A small number of cells show negative or near-zero deltas, typically involving rare
classes with few training examples (e.g., class~8) or class pairs with high mechanistic
overlap where the base model already produces some class-correct samples by chance.

\begin{figure*}[htbp]
    \centering
    \resizebox{0.6\textheight}{!}{
    \includegraphics[width=\textwidth]{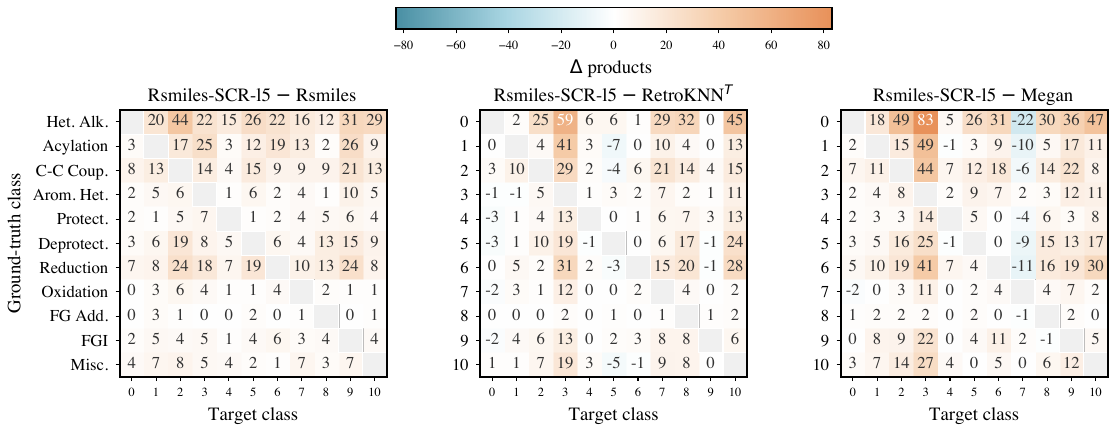}
}
\caption{The difference between guided and baseline results for steering towards all non-ground-truth reaction types.
The heatmaps show the difference in the number of products with precursors from non-ground-truth target classes, stratified per class.
Our guidance scheme leads to improved rates for almost all class pairs, with exceptions discussed in the text.}
\label{fig:delta-heatmaps}
\end{figure*}

\paragraph{Single-step results on USPTO-50k}
We repeat the steering analysis on the full USPTO-50k test set to verify that the
trends observed on USPTO-190-Steps generalize to a larger and more diverse benchmark.
\Cref{fig:steering-breadth-50k} reports the steering breadth of all baselines and
SCR$_{\text{rxn}}$ on USPTO-50k, and \cref{fig:delta-heatmaps-50k} shows the
difference in class-correct product counts between SCR$_{\text{rxn}}$ and
representative baselines (the unguided generator ablation, the strongest template-based
method, and the strongest template-free method).
The results mirror the USPTO-190-Steps findings: SCR$_{\text{rxn}}$ achieves the
highest steering breadth, and the delta heatmaps are predominantly positive across
all (source, target) class pairs.

\begin{figure*}[htbp]
    \centering
    \includegraphics[width=\textwidth]{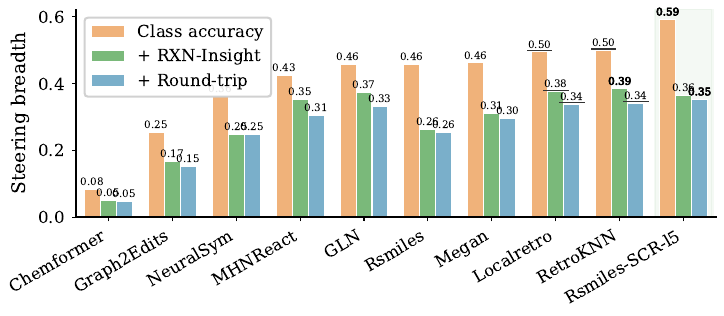}
\caption{Steering breadth on USPTO-50k for all baselines and SCR$_{\text{rxn}}$.
As on USPTO-190-Steps, SCR$_{\text{rxn}}$ achieves the highest steering breadth across
all three criteria (class accuracy, RXN-Insight name match, and round-trip accuracy).}
\label{fig:steering-breadth-50k}
\end{figure*}

\begin{figure*}[htbp]
    \centering
    \resizebox{0.6\textheight}{!}{%
    \includegraphics[width=\textwidth]{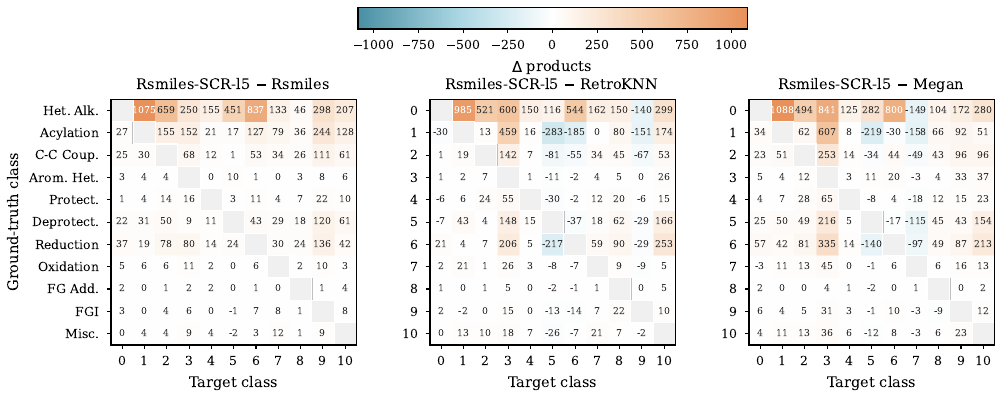}
}%
\caption{Difference between SCR$_{\text{rxn}}$ and baseline results for steering towards all
non-ground-truth reaction types on USPTO-50k. The heatmaps show the difference in the
number of products with precursors from non-ground-truth target classes, stratified per
class, against the unguided generator ablation, RetroKNN (template-based), and Megan
(template-free). As on USPTO-190-Steps, our guidance scheme improves steering rates
for the vast majority of class pairs.}
\label{fig:delta-heatmaps-50k}
\end{figure*}

\subsection{Multi-step synthesis results}
\paragraph{Search-based multi-step Synthesis}
\Cref{tab:search-baselines-all} extends the main-body comparison to all baselines
evaluated with Retro* search on USPTO-190.
With a USPTO-50k checkpoint, S$_{\text{tan}}$ achieves a 95\% solve rate,
far above the best unguided method (Megan at 66\%) and all template-based baselines
(LocalRetro 50\%, RetroKNN 54\%).
Guided methods are also substantially more node-efficient, requiring
${\sim}877$--$1{,}024$ search nodes versus $6{,}000$--$8{,}500$ for
template-based approaches.
With a USPTO-190 checkpoint, both guidance variants reach a 91\% solve rate,
with Tanimoto guidance producing up to 8.22 non-overlapping routes per target
and the highest conditional starting-material rate (35\%).

We further evaluate on two additional benchmarks from the Pistachio dataset.
On \textbf{Pistachio Reachable} (\cref{tab:search-pistachio-reachable}), most methods
already achieve near-ceiling solve rates (95--97\%), so the primary benefit of
guidance is increased route diversity: S$_{\text{tan}}$ discovers
4.45 routes per target compared to 3.86 for the unguided generator, and with
a USPTO-190 checkpoint this grows to 14.22 routes per target---a substantial
diversity gain.
On \textbf{Pistachio Hard} (\cref{tab:search-pistachio-hard}), the benchmark is
more challenging with wider baseline spread (33\%--76\% solve rate).
Tanimoto guidance lifts the generator from 75\% to 89\% (+14\,pp), and with
a USPTO-190 checkpoint reaches 93\% with 10.76 non-overlapping routes per target,
confirming that guidance benefits scale with problem difficulty.

\begin{table}[htb!]
\caption{Search metrics on USPTO-190 for baselines trained on USPTO-50k.}
\label{tab:search-baselines-all}
\centering
\small
\setlength{\tabcolsep}{2pt}
\begin{tabularx}{\columnwidth}{@{}Xccccccccc@{}}
    \toprule
    Method & Rate & Rate w/ SM & N. Routes & R.Trip & Class & Name & Nodes & Calls \\
     & ($\uparrow$) & ($\uparrow$) & ($\uparrow$) & ($\uparrow$) & ($\uparrow$) & ($\uparrow$) & ($\downarrow$) & ($\downarrow$) \\
    \midrule
    \multicolumn{9}{l}{\textit{USPTO-50k}} \\[0pt]
    \quad Localretro & 0.50 & 0.22 & 1.64 & 0.21 & 0.14 & 0.35 & 8075.65 & 39.08 \\
    \quad GLN & 0.24 & 0.11 & 0.39 & 0.12 & 0.06 & 0.18 & 6332.84 & 22.01 \\
    \quad Retroknn & 0.54 & \underline{0.25} & 1.58 & 0.23 & 0.14 & 0.37 & 8013.27 & 37.64 \\
    \quad MHNreact & 0.06 & 0.02 & 0.09 & 0.03 & 0.02 & 0.04 & \textbf{551.09} & \textbf{5.03} \\
    \quad Chemformer & 0.41 & 0.10 & 0.93 & 0.17 & 0.08 & 0.23 & 976.31 & 13.97 \\
    \quad Graph2Edits & 0.54 & 0.20 & 1.19 & 0.25 & 0.12 & 0.35 & 5487.67 & 22.53 \\
    \quad Megan & 0.66 & 0.21 & 1.37 & \textbf{0.27} & 0.13 & \textbf{0.43} & 8505.98 & 34.77 \\
    \quad Generator & 0.17 & 0.01 & 0.26 & 0.06 & 0.02 & 0.08 & 2410.82 & 33.76 \\
    \rowcolor{highlightgreen}
    \quad SCR$_{\text{rxn}}$ & 0.78 & 0.15 & 2.41 & 0.18 & 0.22 & 0.31 & 1023.66 & 7.00 \\
    \rowcolor{highlightgreen}
    \quad SCR$_{\text{tan}}$ & \textbf{0.95} & 0.18 & 3.33 & 0.23 & \textbf{0.24} & \textbf{0.43} & \underline{876.99} & 6.64 \\[0pt]
    \multicolumn{9}{l}{\textit{USPTO-190}} \\[0pt]
    \quad Neuralsym & 0.47 & 0.16 & 0.83 & \underline{0.26} & 0.02 & 0.27 & 17622.49 & 87.96 \\
    \quad Generator & 0.81 & 0.17 & 1.93 & \underline{0.26} & \underline{0.23} & \underline{0.42} & 1071.63 & 9.14 \\
    \rowcolor{highlightgreen}
    \quad SCR$_{\text{rxn}}$ & \underline{0.91} & 0.17 & \underline{4.40} & 0.05 & 0.19 & 0.05 & 1436.35 & \underline{5.58} \\
    \rowcolor{highlightgreen}
    \quad SCR$_{\text{tan}}$ & \underline{0.91} & \textbf{0.35} & \textbf{8.22} & 0.10 & 0.20 & 0.14 & 3613.41 & 9.82 \\
    \bottomrule
\end{tabularx}
\end{table}
\begin{table}[htb!]
\caption{Search metrics on Pistachio Reachable.}
\label{tab:search-pistachio-reachable}
\centering
\small
\setlength{\tabcolsep}{4pt}
\begin{tabular}{lccccccc}
    \toprule
    Method & Rate & Rate w/ SM & N. Routes & R.Trip & Name & Nodes & Calls \\
     & ($\uparrow$) & ($\uparrow$) & ($\uparrow$) & ($\uparrow$) & ($\uparrow$) & ($\downarrow$) & ($\downarrow$) \\
    \midrule
    \multicolumn{8}{l}{\textit{USPTO-50k}} \\[0pt]
    \quad GLN & 0.80 & 0.00 & 2.65 & 0.54 & 0.69 & 6702.69 & 27.65 \\
    \quad MHNreact & 0.54 & 0.21 & 1.43 & 0.36 & 0.46 & \textbf{643.95} & \textbf{4.99} \\
    \quad Graph2Edits & \underline{0.95} & \underline{0.59} & 3.23 & \underline{0.58} & \textbf{0.75} & 5260.79 & 27.93 \\
    \quad Chemformer & 0.87 & 0.40 & 2.08 & 0.51 & 0.64 & \underline{760.91} & 14.67 \\
    \quad Megan & \underline{0.95} & 0.51 & 3.10 & \textbf{0.59} & \underline{0.74} & 9012.99 & 33.34 \\
    \quad Generator & \textbf{0.97} & 0.45 & 3.86 & 0.52 & 0.67 & 1442.49 & 11.80 \\
    \rowcolor{highlightgreen}
    \quad SCR$_{\text{tan}}$ & \textbf{0.97} & 0.35 & \underline{4.45} & 0.49 & 0.60 & 943.20 & \underline{7.38} \\[0pt]
    \multicolumn{8}{l}{\textit{USPTO-190}} \\[0pt]
    \quad NeuralSym & \underline{0.95} & \textbf{0.67} & 3.60 & \textbf{0.59} & 0.73 & 9349.81 & 53.54 \\
    \quad Generator & \textbf{0.97} & \textbf{0.67} & 3.95 & 0.50 & 0.59 & 2083.88 & 13.29 \\
    \rowcolor{highlightgreen}
    \quad SCR$_{\text{tan}}$ & \textbf{0.97} & 0.44 & \textbf{14.22} & 0.17 & 0.21 & 4295.86 & 10.36 \\
    \bottomrule
\end{tabular}
\end{table}
\begin{table}[htb!]
\caption{Search metrics on Pistachio Hard.}
\label{tab:search-pistachio-hard}
\centering
\small
\setlength{\tabcolsep}{4pt}
\begin{tabular}{lccccccc}
    \toprule
    Method & Rate & Rate w/ SM & N. Routes & R.Trip & Name & Nodes & Calls \\
     & ($\uparrow$) & ($\uparrow$) & ($\uparrow$) & ($\uparrow$) & ($\uparrow$) & ($\downarrow$) & ($\downarrow$) \\
    \midrule
    \multicolumn{8}{l}{\textit{USPTO-50k}} \\[0pt]
    \quad GLN & 0.35 & 0.08 & 0.83 & 0.21 & 0.28 & 8789.96 & \textbf{0.00} \\
    \quad Localretro & 0.59 & \underline{0.21} & 2.58 & 0.31 & 0.45 & 8606.85 & 39.53 \\
    \quad Retroknn & 0.59 & 0.20 & 2.41 & 0.33 & 0.45 & 8685.44 & \textbf{0.00} \\
    \quad MHNreact & 0.33 & 0.06 & 0.91 & 0.21 & 0.29 & 1916.76 & 13.78 \\
    \quad Graph2Edits & 0.76 & 0.20 & 1.91 & \textbf{0.37} & \textbf{0.54} & 5110.46 & 23.78 \\
    \quad Megan & 0.70 & 0.16 & 2.03 & \underline{0.34} & \underline{0.48} & 8562.38 & 34.17 \\
    \quad Chemformer & 0.47 & 0.13 & 1.16 & 0.23 & 0.28 & \underline{919.84} & 14.83 \\
    \quad Generator & 0.75 & 0.15 & 2.34 & 0.27 & 0.45 & \textbf{908.60} & 8.05 \\
    \rowcolor{highlightgreen}
    \quad SCR$_{\text{tan}}$ & \underline{0.89} & 0.15 & \underline{3.52} & 0.26 & 0.46 & 995.44 & \underline{7.25} \\[0pt]
    \multicolumn{8}{l}{\textit{USPTO-190}} \\[0pt]
    \quad NeuralSym & 0.55 & \underline{0.21} & 1.98 & 0.31 & 0.39 & 8646.33 & 50.84 \\
    \quad Generator & 0.82 & \textbf{0.31} & 3.15 & 0.32 & 0.46 & 2159.37 & 13.94 \\
    \rowcolor{highlightgreen}
    \quad SCR$_{\text{tan}}$ & \textbf{0.93} & 0.20 & \textbf{10.76} & 0.10 & 0.18 & 4088.90 & 11.05 \\
    \bottomrule
\end{tabular}
\end{table}

\paragraph{On the complementarity of token-guidance and search-level guidance}
We decompose the contribution of three guidance mechanisms:
\textbf{S}~(token-level SCR reranking during beam search),
\textbf{G}~(search-level guidance that modifies the Retro* value function),
and \textbf{F}~(hard filtering that discards candidates not satisfying the target property).
\Cref{tab:search-metrics-combined} reports all individual mechanisms and their combinations.

Filtering alone (F$_{\text{rxn}}$, 21\%) provides only a marginal improvement over
the unguided generator (17\%), because discarding candidates too aggressively limits the
search frontier and prevents the algorithm from finding viable routes.
Search-level guidance alone (G$_{\text{rxn}}$, 76\%) nearly matches token-level
guidance (S$_{\text{rxn}}$, 78\%), confirming that both levels of intervention
are independently effective.
Combining token-level and search-level guidance (SG$_{\text{rxn}}$) yields a 75\%
solve rate---a 3\,pp drop relative to S$_{\text{rxn}}$ alone (78\%)---with marginally
higher route diversity (2.44 vs.\ 2.41 non-overlapping routes per target). The
two mechanisms therefore appear to act on overlapping rather than fully orthogonal
failure modes: token-level guidance already diversifies the single-step candidates,
and adding search-level guidance on top trades a small amount of solve rate for a
small gain in route diversity, rather than producing a clear additive benefit.
However, adding filtering to any combination consistently reduces performance
(SF: 27\%, FG: 24\%, SFG: 25\%), as the hard constraint eliminates candidates
that might lead to viable routes through intermediate steps.
Filtering is therefore best reserved for settings where the chemist wants to
strictly enforce a property constraint at the expense of exploration breadth.

\begin{table}[htb!]
\caption{Search metrics on USPTO-190 for baselines trained on USPTO-50k.}
\label{tab:search-metrics-combined}
\centering
\small
\setlength{\tabcolsep}{2pt}
\begin{tabularx}{\columnwidth}{@{}Xccccccccc@{}}
    \toprule
    Method & Rate & Rate w/ SM & N. Routes & R.Trip & Class & Name & Nodes & Calls \\
     & ($\uparrow$) & ($\uparrow$) & ($\uparrow$) & ($\uparrow$) & ($\uparrow$) & ($\uparrow$) & ($\downarrow$) & ($\downarrow$) \\
    \midrule
    \multicolumn{9}{l}{\textit{Guided}} \\[0pt]
    \quad Generator & 0.17 & 0.01 & 0.26 & 0.06 & 0.02 & 0.08 & 2410.82 & 33.76 \\
    \quad F$_{\text{rxn}}$ & 0.21 & 0.08 & 0.50 & 0.07 & 0.20 & 0.09 & \textbf{143.82} & 7.19 \\
    \quad G$_{\text{rxn}}$ & \underline{0.76} & \textbf{0.15} & 1.97 & \textbf{0.25} & \underline{0.25} & \textbf{0.37} & 923.07 & 8.07 \\
    \quad FG$_{\text{rxn}}$ & 0.24 & 0.00 & 0.70 & 0.07 & 0.00 & 0.09 & \underline{157.91} & 7.38 \\
    \rowcolor{highlightgreen}
    \quad S$_{\text{rxn}}$ & \textbf{0.78} & \textbf{0.15} & \underline{2.41} & \underline{0.18} & 0.22 & \underline{0.31} & 1023.66 & 7.00 \\
    \rowcolor{highlightgreen}
    \quad SF$_{\text{rxn}}$ & 0.27 & 0.07 & 0.76 & 0.06 & \textbf{0.26} & 0.09 & 190.48 & 6.97 \\
    \rowcolor{highlightgreen}
    \quad SG$_{\text{rxn}}$ & 0.75 & \underline{0.12} & \textbf{2.44} & 0.16 & 0.24 & 0.29 & 1001.07 & \underline{6.96} \\
    \rowcolor{highlightgreen}
    \quad SFG$_{\text{rxn}}$ & 0.25 & 0.07 & 0.67 & 0.05 & 0.24 & 0.07 & 199.72 & \textbf{6.72} \\
    \bottomrule
\end{tabularx}
\end{table}

\paragraph{Hard filtering generalizes poorly across baselines.}
The unguided-generator ablation in \cref{tab:search-metrics-combined} suggests that hard
filtering by reaction type is too aggressive a constraint, but one might wonder
whether this is specific to our generator.
\Cref{tab:search-filtered-baselines} repeats the F$_{\text{rxn}}$ intervention
for all eight retrosynthesis baselines (USPTO-50k checkpoints) and contrasts
them with token-level steering (SCR$_{\text{rxn}}$).
The pattern is consistent: every filtered baseline collapses to a solve rate
between 0\% (Megan-F$_{\text{rxn}}$, on which our search jobs only completed for
1/190 targets) and 20\% (F$_{\text{rxn}}$), with template-based
methods that nominally produce class-correct candidates by construction
(GLN, LocalRetro, RetroKNN, NeuralSym) faring no better than template-free
sequence/graph models (Graph2Edits, Chemformer).
Discarding any candidate whose predicted class disagrees with the target
prunes the search frontier so aggressively that Retro* cannot stitch together
multi-step routes regardless of the underlying model family.
Token-level steering, in contrast, only re-weights the candidate distribution
without removing options, so the search retains enough breadth to find routes:
SCR$_{\text{rxn}}$ reaches a 78\% solve rate---roughly $4{\times}$ the
best filtered baseline---while still producing 2.41 non-overlapping routes per
target.
This confirms that the failure mode of hard filtering is structural rather than
model-specific, and that soft, token-level guidance is the right granularity at
which to enforce a property constraint during search.

\begin{table}[htb!]
\caption{Search metrics on USPTO-190 for filtered retrosynthesis baselines (USPTO-50k checkpoints), with token-level steered Generator (SCR$_{\text{rxn}}$) included for comparison. Hard filtering caps every baseline well below a 20\% solve rate, while soft token-level steering reaches 78\%.}
\label{tab:search-filtered-baselines}
\centering
\small
\setlength{\tabcolsep}{2pt}
\begin{tabularx}{\columnwidth}{@{}Xccccccccc@{}}
    \toprule
    Method & Rate & Rate w/ SM & N. Routes & R.Trip & Class & Name & Nodes & Calls \\
     & ($\uparrow$) & ($\uparrow$) & ($\uparrow$) & ($\uparrow$) & ($\uparrow$) & ($\uparrow$) & ($\downarrow$) & ($\downarrow$) \\
    \midrule
    \multicolumn{9}{l}{\textit{USPTO-50k}} \\[0pt]
    \quad F$_{\text{rxn}}$ & \underline{0.20} & \underline{0.09} & \underline{0.47} & 0.06 & \underline{0.18} & 0.09 & \textbf{146.22} & 7.34 \\
    \quad NeuralSym-F$_{\text{rxn}}$ & 0.14 & 0.08 & 0.15 & \underline{0.09} & 0.10 & 0.07 & 497.62 & 36.09 \\
    \quad Megan-F$_{\text{rxn}}$ & 0.00 & 0.00 & 0.00 & 0.00 & 0.00 & 0.00 & 245.00 & 57.00 \\
    \quad Graph2Edits-F$_{\text{rxn}}$ & 0.15 & 0.07 & 0.28 & 0.04 & 0.13 & 0.08 & 393.13 & 18.84 \\
    \quad Chemformer-F$_{\text{rxn}}$ & 0.06 & 0.02 & 0.15 & 0.02 & 0.05 & 0.02 & \underline{207.56} & \underline{6.16} \\
    \quad Retroknn-F$_{\text{rxn}}$ & 0.16 & 0.08 & 0.31 & 0.07 & 0.13 & \underline{0.10} & 845.06 & \textbf{0.00} \\
    \quad Localretro-F$_{\text{rxn}}$ & 0.15 & 0.08 & 0.24 & 0.07 & 0.12 & 0.08 & 829.18 & 38.38 \\
    \quad GLN-F$_{\text{rxn}}$ & 0.04 & 0.02 & 0.04 & 0.01 & 0.04 & 0.02 & 653.96 & \textbf{0.00} \\
    \rowcolor{highlightgreen}
    \quad SCR$_{\text{rxn}}$ & \textbf{0.78} & \textbf{0.15} & \textbf{2.41} & \textbf{0.18} & \textbf{0.22} & \textbf{0.31} & 1023.66 & 7.00 \\
    \bottomrule
\end{tabularx}
\end{table}

\paragraph{Lookahead provides minimal benefit over fixed-scale guidance.}
As shown in \cref{tab:lookahead-ablation-search}, direct guidance with
$\lambda{=}1.0$ (SCR-nla) achieves a solve rate of 78.4\%, substantially
outperforming the standard lookahead variant (42.6\%) while using comparable
wall-clock time (643.6\,s vs 734.3\,s).
The lookahead algorithm's exploration phase consumes model calls on suboptimal
$\lambda$ values, leaving fewer calls for exploitation.
Lookahead-fast recovers to 78.9\% by exhaustively evaluating all parameters,
but at $5.8{\times}$ the computational cost (3758.9\,s).
Route quality follows the same pattern
(\cref{tab:lookahead-ablation-quality}): SCR-nla produces 2.41 non-overlapping
routes per target with 17.9\% round-trip accuracy, compared to 0.83 routes and
6.4\% for SCR-la.

Per-target analysis reveals that all 81 targets solved by SCR-la are also solved
by the union of SCR-nla variants, with zero targets uniquely solved by lookahead.
These results indicate that for reaction-type guidance with $\lambda{=}1.0$,
the fixed-scale strategy is both simpler and more effective than adaptive
parameter selection, and we adopt it as the default in our main results.

\begin{table}[htb!]
\caption{Search metrics: lookahead ablation on USPTO-190.}
\label{tab:lookahead-ablation-search}
\centering
\small
\setlength{\tabcolsep}{2pt}
\begin{tabularx}{\columnwidth}{@{}Xccccccccc@{}}
    \toprule
    Method & Rate & Rate w/ SM & N. Routes & R.Trip & Class & Name & Nodes & Calls \\
     & ($\uparrow$) & ($\uparrow$) & ($\uparrow$) & ($\uparrow$) & ($\uparrow$) & ($\uparrow$) & ($\downarrow$) & ($\downarrow$) \\
    \midrule
    \multicolumn{9}{l}{\textit{USPTO-50k}} \\[0pt]
    \quad Generator (no guidance) & 0.17 & 0.01 & 0.26 & 0.06 & 0.02 & 0.08 & 2410.82 & 33.76 \\
    \rowcolor{highlightgreen}
    \quad SCR-nla g=1.0 & \underline{0.78} & \underline{0.15} & \underline{2.41} & \underline{0.18} & \textbf{0.22} & \underline{0.31} & \textbf{1023.66} & \textbf{7.00} \\
    \rowcolor{highlightgreen}
    \quad SCR-la-fast & \textbf{0.79} & \textbf{0.17} & \textbf{2.42} & \textbf{0.20} & \underline{0.21} & \textbf{0.33} & \underline{1804.40} & \underline{11.93} \\
    \bottomrule
\end{tabularx}
\end{table}
\begin{table}[htb!]
\caption{Quality of routes: lookahead ablation on USPTO-190.}
\label{tab:lookahead-ablation-quality}
\centering
\small
\setlength{\tabcolsep}{4pt}
\begin{tabular}{lcccc}
    \toprule
    Method & Match & R.Trip & Class & Name \\
     & ($\uparrow$) & ($\uparrow$) & ($\uparrow$) & ($\uparrow$) \\
    \midrule
    \multicolumn{5}{l}{\textit{USPTO-50k}} \\[2pt]
    \quad Generator (no guidance) & \underline{0.00} & 0.06 & 0.02 & 0.08 \\
    \rowcolor{highlightgreen}
    \quad SCR-nla g=1.0 & \textbf{0.04} & \underline{0.18} & \textbf{0.22} & \underline{0.31} \\
    \quad SCR-la-fast g=1 (rxntype) & \textbf{0.04} & \textbf{0.20} & \underline{0.21} & \textbf{0.33} \\
    \bottomrule
\end{tabular}
\end{table}

\paragraph{Qualitative analysis of steered routes}
We present detailed case studies in \cref{fig:case-study}, where guidance recovers ground-truth precursors and complete routes that no baseline finds.

\textbf{Wieland-Gumlich Aldehyde (\cref{fig:case-study}a).}
Wieland-Gumlich Aldehyde is an essential intermediate in all known synthesis pathways of Strychnine, a historically significant alkaloid \citep{Genheden2025RouteSimilarity}.
Two notable retrosynthetic disconnections for the aldehyde are deprotection (class~5) and reduction (class~6). While the unguided generator recovers only the reduction precursor, guiding toward deprotection specifically recovers the ground-truth TMS-ether precursor, demonstrating that token-level guidance can unlock alternative disconnection strategies that the base model alone cannot access.

\textbf{Target~49: rxntype guidance (\cref{fig:case-study}b).}
All eight baselines fail to find any valid route for Target~49 despite exploring up to $7{,}483$ search nodes; rxntype guidance steers the model toward a 3-step route via Protection and C-C Coupling, matching the ground-truth reaction class in 2 of 3 steps.

\textbf{Target~169: Tanimoto guidance (\cref{fig:case-study}c).}
Target~169 is likewise unsolved by every baseline. Tanimoto guidance discovers a 3-step route in which the similarity to the designated starting material (penicillamine) increases along the retrosynthetic steps ($0.15 \!\to\! 0.17 \!\to\! 0.60$), reaching a close structural analog at the leaf node.

\textbf{Aggregate.}
Across the full dataset, rxntype and Tanimoto guidance jointly unlock routes for 33 out of 190 targets (17.4\%) that remain inaccessible to all baselines, producing 94 unique routes for the 20 targets solvable only by rxntype alone.

\subsection{Additional empirical justification for token-guidance}
\label{app:token-guidance-needed-additional}

\paragraph{Increasing the beam size achieves diversity at the cost of quality.} \Cref{tab:beamsize_ablation}
compares the unguided generator at beam sizes 10 and 20 against SCR$_{\text{rxn}}$ (guided, beam=10), holding the sample budget fixed at
100 per product. Widening the beam from 10 to 20 improves raw steering rate
(+13.7\%) but degrades rank efficiency: mean rank worsens from 12.1 to 21.0
and top-10 coverage drops from 54.6\% to 42.3\%, as useful predictions are
buried deeper among an inflated candidate set. SCR$_{\text{rxn}}$ matches the steering
rate of unguided beam=20 (63.6 vs.\ 63.2\%) while recovering class-correct
predictions at mean rank 9.0 vs.\ 21.0, with top-10 coverage of 69.0\%
vs.\ 42.3\%. Guidance thus recovers the diversity gains of a wider beam
without the associated rank degradation. 
\begin{figure}[h]
    \centering
    \begin{minipage}{\textwidth}
        \captionof{table}{%
  Beam size ablation on 640 USPTO-190 test products over 11 reaction types. SCR$_{\text{rxn}}$ uses beam=10 with guidance; unguided baselines vary beam width only. \textit{Steer.}\ counts any class-correct prediction in 100 samples; \textit{+RXN-I} further requires RXN-Insight validation; \textit{+RT} additionally requires round-trip accuracy (all in \%). Mean rank is computed over products where at least one class-correct prediction was found.
}
\label{tab:beamsize_ablation}
\vspace{4pt}
\centering
\small
\setlength{\tabcolsep}{5pt}
\begin{tabular}{lcccccc}
\toprule
& \multicolumn{3}{c}{\textit{Steering success (\%)}} & \multicolumn{3}{c}{\textit{Rank efficiency}} \\
\cmidrule(lr){2-4} \cmidrule(lr){5-7}
\textbf{Method} & \textbf{Steer.} & \textbf{+RXN-I} & \textbf{+RT} & \textbf{Mean rank} & \textbf{Top-5 (\%)} & \textbf{Top-10 (\%)} \\
\midrule
Generator (beam=10) & 49.5 & 28.9 & 18.6 & \underline{12.1} & \underline{35.8} & \underline{54.6} \\
Generator (beam=20) & \underline{63.2} & \underline{36.0} & \underline{23.7} & 21.0 & 27.9 & 42.3 \\
\midrule
\rowcolor{highlightgreen} \textbf{SCR$_{\text{rxn}}$} & \textbf{63.6} & \textbf{39.2} & \textbf{24.7} & \phantom{0}\textbf{9.0} & \textbf{43.9} & \textbf{69.0} \\
\bottomrule
\end{tabular}
    \end{minipage}
\end{figure}

\paragraph{Generating more samples is inefficient and ineffective.} \Cref{tab:samplesize_ablation} compares SCR$_{\text{rxn}}$ at 100 samples against template-based and template-free baselines at 600 samples. With the larger budget the unguided generator does reach the highest raw steering breadth (69.9\%), but its class-correct predictions are spread across the full ranked list (mean rank 28.3, top-10 38.8\%); SCR$_{\text{rxn}}$ matches or beats all baselines on rank efficiency despite using $6\times$ fewer samples (mean rank 9.0, top-10 69.0\%) and recovers a class-correct prediction within the top 100 for every product where one exists.
\begin{figure}[h]
    \centering
    \begin{minipage}{\textwidth}
        \captionof{table}{%
  Sample size ablation on 640 USPTO-190 test products over 11 reaction types. SCR$_{\text{rxn}}$ uses 100 samples with guidance; unguided baselines use 600 samples. \textit{Steer.}\ counts any class-correct prediction; \textit{+RXN-I} further requires RXN-Insight validation; \textit{+RT} additionally requires round-trip accuracy (all in \%). Mean rank is computed over products where at least one class-correct prediction was found.
}
\label{tab:samplesize_ablation}
\vspace{4pt}
\centering
\small
\setlength{\tabcolsep}{5pt}
\begin{tabular}{lcccccc}
\toprule
& \multicolumn{3}{c}{\textit{Steering success (\%)}} & \multicolumn{3}{c}{\textit{Rank efficiency}} \\
\cmidrule(lr){2-4} \cmidrule(lr){5-7}
\textbf{Method} & \textbf{Steer.} & \textbf{+RXN-I} & \textbf{+RT} & \textbf{Mean rank} & \textbf{Top-5 (\%)} & \textbf{Top-10 (\%)} \\
\midrule
Localretro-600 & 62.7 & \underline{48.8} & \underline{29.1} & 28.8 & 27.4 & 42.1 \\
RetroKNN-600 & \underline{63.6} & \textbf{50.3} & \textbf{29.7} & 29.1 & 27.9 & 42.2 \\
MHNReact-600 & 45.5 & 37.8 & 21.5 & 15.2 & 36.7 & 54.6 \\
Megan-600 & 61.0 & 39.5 & 25.2 & 35.9 & 27.0 & 40.3 \\
Graph2Edits-600 & 39.4 & 21.8 & 13.9 & 19.8 & 43.0 & 61.3 \\
Chemformer-600 & 28.1 & 13.8 & 10.2 & \underline{11.5} & \textbf{47.8} & \underline{67.0} \\
Generator-600 & \textbf{69.9} & 39.4 & 26.0 & 28.3 & 25.2 & 38.8 \\
\midrule
\rowcolor{highlightgreen} \textbf{SCR$_{\text{rxn}}$} & \underline{63.6} & 39.2 & 24.7 & \phantom{0}\textbf{9.0} & \underline{43.9} & \textbf{69.0} \\
\bottomrule
\end{tabular}
    \end{minipage}
\end{figure}

\subsection{Ablations}
\label{app:ablations}

We provide here the full details behind the ablation results summarised in \cref{sec:ablations}. All experiments use reaction-type steering on USPTO-190-steps (640 products, 100 samples per product). Each ablation varies one hyperparameter while keeping the others at their default values ($\lambda{=}1.0$, $N{=}72$, $L_{\min}{=}5$).

\begin{figure}[t]
    \centering
    \begin{subfigure}[t]{0.48\textwidth}
        \centering
        \includegraphics[width=\textwidth]{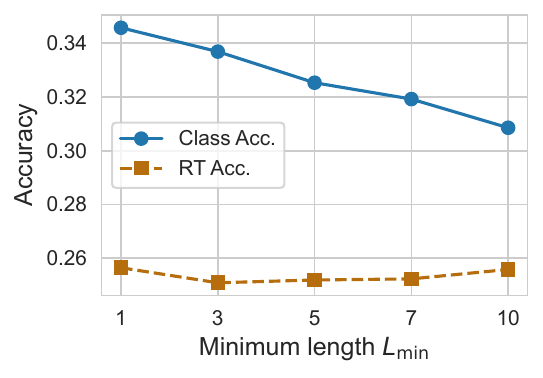}
        \caption{Guidance onset length $L_{\min}$}
        \label{fig:ablation-lmin}
    \end{subfigure}
    \hfill
    \begin{subfigure}[t]{0.48\textwidth}
        \centering
        \includegraphics[width=\textwidth]{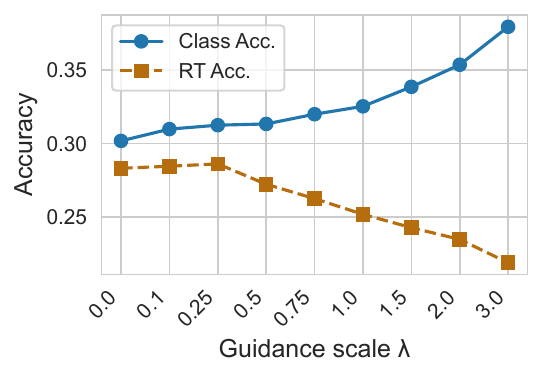}
        \caption{Guidance scale $\lambda$}
        \label{fig:ablation-scale}
    \end{subfigure}
    \caption{Effect of the two main hyperparameters on class accuracy (steering, blue) and round-trip accuracy (validity, orange). \textbf{(a)}~Guidance onset $L_{\min}$: earlier guidance improves class accuracy by ${\sim}3.7$\,pp without degrading round-trip accuracy. \textbf{(b)}~Guidance scale $\lambda$: class accuracy increases by $7.7$\,pp while round-trip accuracy drops by $6.4$\,pp.}
    \label{fig:ablation-combined}
\end{figure}

\paragraph{Number of candidate tokens ($N$)}
We evaluate $N \in \{10, 24, 48, 72\}$. As noted in \cref{sec:ablations}, all metrics are effectively constant across this range: class accuracy varies by less than $0.2$\,pp, and round-trip accuracy by less than $0.2$\,pp. This indicates that the ranking step successfully identifies the best tokens even from a small candidate pool, and that $N$ can be reduced from $72$ to $10$ --- yielding a proportional speedup in the classifier forward passes --- with no measurable loss in quality.

\paragraph{Guidance onset length ($L_{\min}$)}
We evaluate $L_{\min} \in \{1, 3, 5, 7, 10\}$. $L_{\min}$ controls the number of tokens generated by the unguided base model before classifier guidance is applied. Lower $L_{\min}$ means guidance starts earlier. As shown in \cref{fig:ablation-combined}(a), earlier guidance onset monotonically improves class accuracy (from $30.9\%$ at $L_{\min}{=}10$ to $34.6\%$ at $L_{\min}{=}1$, a ${\sim}3.7$\,pp improvement), while round-trip accuracy remains essentially unchanged ($\Delta < 0.6$\,pp). However, guiding from the very first token ($L_{\min}{=}1$) slightly reduces the number of surviving samples, likely because early-token guidance can occasionally push the model toward degenerate sequences.

\paragraph{Guidance scale ($\lambda$)}
We evaluate $\lambda \in \{0.0, 0.1, 0.25, 0.5, 0.75, 1.0, 1.5, 2.0, 3.0\}$. This is the only hyperparameter that produces a meaningful trade-off (\cref{fig:ablation-combined}(b)): as $\lambda$ increases, class accuracy rises from $30.2\%$ to $37.9\%$ ($+7.7$\,pp) while round-trip accuracy drops from $28.3\%$ to $21.9\%$ ($-6.4$\,pp). The divergence between these two curves demonstrates that stronger guidance biases the beam toward class-correct but chemically less faithful sequences. This clear steering--validity trade-off is the central motivation for the adaptive lookahead algorithm (\cref{app:lookahead-inference}), which adjusts $\lambda$ per product to navigate the Pareto frontier.


\end{document}